%% file: main.tex
\documentclass[review,onefignum,onetabnum,11pt]{siamart171218}
\renewcommand{\footercopyright}{}
\input{ex_shared}
\let\tilde\widetilde
\nolinenumbers
\begin{document}

\maketitle

\begin{abstract}
\textcolor{black}{Given a function dictionary $\mathcal{D}$ and an approximation budget $N\in\mathbb{N}^+$,  nonlinear approximation seeks the linear combination of the best $N$ terms $\{T_n\}_{1\leq n\leq N}\subseteq \mathcal{D}$ to approximate a given function $f$ with the minimum approximation error \[\epsilon_{L,f}:=\min_{\{g_n\}\subseteq{\mathbb{R}},\{T_n\}\subseteq\mathcal{D}} \|f(\bm{x})-\sum_{n=1}^N g_n T_n(\bm{x})\|.\] Motivated by recent success of deep learning, we propose dictionaries with functions in a form of compositions, i.e., \[T(\bm{x})=T^{(L)}\circ T^{(L-1)}\circ\cdots\circ T^{(1)}(\bm{x})\] for all $T\in\mathcal{D}$, and implement $T$ using ReLU feed-forward neural networks (FNNs) with $L$ hidden layers. We further quantify the improvement of the best $N$-term approximation rate in terms of $N$ when $L$ is increased from $1$ to $2$ or $3$ to show the power of compositions. In the case when $L>3$, our analysis shows that increasing $L$ cannot improve the approximation rate in terms of $N$.}

\textcolor{black}{In particular, for any function $f$ on $[0,1]$, regardless of its smoothness and even the continuity, if $f$ can be approximated using a dictionary when $L=1$ with the best $N$-term approximation rate $\epsilon_{L,f}=\OO(N^{-\eta})$, we show that dictionaries with $L=2$ can improve the best $N$-term approximation rate to $\epsilon_{L,f}=\OO(N^{-2\eta})$. We also show that for H{\"o}lder continuous functions of order $\alpha$ on $[0,1]^d$, the application of a dictionary with $L=3$ in nonlinear approximation can achieve an essentially tight best $N$-term approximation rate $\epsilon_{L,f}=\OO( N^{-2\alpha/d})$. Finally, we show that dictionaries consisting of wide FNNs with a few hidden layers are more attractive in terms of computational efficiency than dictionaries with narrow and very deep FNNs for approximating H{\"o}lder continuous functions if the number of computer cores is larger than $N$ in parallel computing. }
\end{abstract}

\begin{keywords}
  Deep Neural Networks, ReLU Activation Function, Nonlinear Approximation, Function Composition, H{\"o}lder Continuity, Parallel Computing.
\end{keywords}


\section{Introduction}

For non-smooth and high-dimensional function approximation, a favorable technique popularized in recent decades is the nonlinear approximation \cite{devore_1998} that does not limit the approximants to come from linear spaces, obtaining sparser representation, cheaper computation, and more robust estimation, and therein emerged the bloom of many breakthroughs in applied mathematics and computer science (e.g., wavelet analysis \cite{doi:10.1137/1.9781611970104}, dictionary learning \cite{DBLP:journals/corr/TariyalMSV16}, data compression and denoising \cite{600614,374249}, adaptive pursuit \cite{Davis94adaptivenonlinear,6810285}, compressed sensing \cite{1614066,4472240}).  

Typically, nonlinear approximation is a two-stage algorithm that designs a good redundant nonlinear dictionary, $\mathcal{D}$, in its first stage, and identifies the optimal approximant as a linear combination of $N$ elements of $\mathcal{D}$ in the second stage:
\begin{equation}
\label{eqn:nla}
f(\bm{x}) \approx g\circ T(\bm{x}) :=\sum_{n=1}^N g_n T_n(\bm{x}),
\end{equation}
where $f(\bm{x})$ is the target function in a Hilbert space $\mathcal{H}$ associated with a norm denoted as $\|\cdot\|_*$, $\{T_n\}\subseteq\mathcal{D}\subseteq\mathcal{H}$, $T$ is a nonlinear map from $\mathbb{R}^d$ to $\mathbb{R}^N$ with the $n$-th coordinate being $T_n$, and $g$ is a linear map from $\mathbb{R}^N$ to $\mathbb{R}$ with the $n$-th coordinate being $g_n\in\mathbb{R}$. The nonlinear approximation seeks $g$ and $T$ such that
\begin{equation}\label{eqn:min}
\{\{T_n\},\{g_n\}\}=\argmin_{\{g_n\}\subseteq{\mathbb{R}},\{T_n\}\subseteq\mathcal{D}}\|f(\bm{x}) - \sum_{n=1}^N g_n T_n(\bm{x})\|_*,
\end{equation}
which is also called the best $N$-term approximation. 
One remarkable approach of nonlinear approximation is based on one-hidden-layer neural networks that give simple and elegant bases of the form $T(\bm{x})=\sigma(\bm{W}\bm{x}+\bm{b})$, where $\bm{W}\bm{x}+\bm{b}$ is a linear transform in $\bm{x}$ with the transformation matrix $\bm{W}$ (named as the weight matrix) and a shifting vector $\bm{b}$ (called bias), and $\sigma$ is a nonlinear function (called the activation function). The approximation
\[
f(\bm{x})\approx \sum_{n=1}^N g_n T_n(\bm{x})= \sum_{n=1}^N g_n\sigma(\bm{W}_n \bm{x} + \bm{b}_n)
\]
includes wavelets pursuits \cite{258082,471413}, adaptive splines \cite{devore_1998,PETRUSHEV2003158}, radial basis functions \cite{radiusbase,HANGELBROEK2010203,Xie2013},  {sigmoidal neural networks \cite{Sig1,Sig2,Sig3,Sig4,Sig5,Cybenko1989ApproximationBS,HORNIK1989359,barron1993}}, etc. For functions in Besov spaces with smoothness $s$, \cite{radiusbase,HANGELBROEK2010203} constructed an $\OO(N^{-s/d})$\footnote{In this paper, we use the big $\OO(\cdot)$ notation when we only care about the scaling in terms of the variables inside $(\cdot)$ and the prefactor outside $(\cdot)$ is independent of these variables.} approximation that is almost optimal \cite{Lin2014} and the smoothness cannot be reduced generally \cite{HANGELBROEK2010203}. For H{\"o}lder continuous functions of order $1$ on $[0,1]^d$, \cite{Xie2013} essentially constructed an $\OO(N^{-\frac{1}{2d}})$ approximation, which is far from the lower bound $\OO(N^{-2/d})$ as we shall prove in this paper. Achieving the optimal approximation rate of general continuous functions in constructive approximation, especially in high dimensional spaces, remains an unsolved challenging problem.

\subsection{Problem Statement}


\textcolor{black}{ReLU FNNs have been proved to be a powerful tool in many fields from various points of view \cite{NIPS2014_5422, 6697897,Bartlett98almostlinear,Sakurai,pmlr-v65-harvey17a,Kearns,Anthony:2009,PETERSEN2018296}, which motivates us to tackle the open problem above via function compositions in the nonlinear approximation using deep ReLU FNNs, i.e., 
\begin{equation}
\label{eqn:nla2}
f(\bm{x}) \approx g\circ T^{(L)}\circ T^{(L-1)}\circ\cdots\circ T^{(1)}(\bm{x}),
\end{equation}
where $T^{(i)}(\bm{x})=\sigma\left(\bm{W}_i \bm{x} + \bm{b}_i\right)$ with $\bm{W}_i\in \R^{N_{i}\times N_{i-1}}$, $\bm{b}_i\in \R^{N_i}$ for $i=1,\dots,L$, $\sigma$ is the ReLU activation function, and $f$ is a H{\"o}lder continuous function. For the convenience of analysis, we consider $N_i=N$ for $i=1,\dots,L$. Let $\mathcal{D}_L$ be the dictionary consisting of ReLU FNNs $g\circ T^{(L)}\circ T^{(L-1)}\circ\cdots\circ T^{(1)}(\bm{x})$ with width $N$ and depth $L$. 
To identify the optimal FNN to approximate $f(\bm{x})$, it is sufficient to solve the following optimization problem
\begin{equation}
\label{eqn:nnopt}
\phi^*=\argmin_{\phi\in\mathcal{D}_L}\|f - \phi \|_*.
\end{equation}} 

\textcolor{black}{The fundamental limit of nonlinear approximation via the proposed dictionary is essentially determined by the approximation power of function compositions in \eqref{eqn:nla2}, which gives a performance guarantee of the minimizer in \eqref{eqn:nnopt}. Since function compositions are implemented via ReLU FNNs, the remaining problem is to quantify the approximation capacity of deep ReLU FNNs, especially their ability to improve the best $N$-term approximation rate in $N$ for any fixed $L$ defined as
\begin{equation}\label{eqn:min2}
\epsilon_{L,f}(N)=\min_{\phi\in\mathcal{D}_L}\|f - \phi \|_*.
\end{equation}}

\textcolor{black}{Function compositions can significantly enrich the dictionary of nonlinear approximation and this idea was not considered in the literature previously due to the expensive computation of function compositions in solving the minimization problem in \eqref{eqn:nnopt}. Fortunately, recent development of efficient algorithms for optimization with compositions (e.g., backpropagation techniques  \cite{werbos1975beyond,Fukushima1980,rumelhart1986psychological} and parallel computing techniques \cite{10.1007/978-3-642-15825-4_10,Ciresan:2011:FHP:2283516.2283603}) makes it possible to explore the proposed dictionary in this paper. Furthermore, with advanced optimization algorithms \cite{Duchi:2011:ASM:1953048.2021068,Johnson:2013:ASG:2999611.2999647,ADAM}, good local minima of \eqref{eqn:nnopt} can be identified efficiently \cite{NIPS2016_6112,DBLP:journals/corr/NguyenH17,opt}.}

\subsection{Related work and contribution}
\textcolor{black}{The main goal in the remaining article is to quantify the best $N$-term approximation rate $\epsilon_{L,f}(N)$ defined in \eqref{eqn:min2} for ReLU FNNs in the dictionary $\mathcal{D}_L$ with a fixed depth $L$ when $f$ is a H{\"o}lder continuous function. This topic is related to several existing approximation theories in the literature, but none of these existing works can be applied to answer the problem addressed in this paper.}

\textcolor{black}{First of all, this paper identifies explicit formulas for the best $N$-term approximation rate 
\begin{equation}
\label{eqn:exf}
\epsilon_{L,f}(N)\leq 
 \begin{cases}
2 \nu N^{-2\alpha}, &\text{when  $L\geq 2$ and $d=1$,}\\
2(2\sqrt{d})^\alpha \nu N^{-2\alpha/d}, &\text{when $L\geq 3$ and  $d>1$},
\end{cases}
\end{equation}
for any $N\in\mathbb{N}^+$ and a H{\"o}lder continuous function $f$ of order $\alpha$ with a constant $\nu$, while existing theories \cite{NIPS2017_7203,bandlimit,yarotsky2017,DBLP:journals/corr/LiangS16,Hadrien,suzuki2018adaptivity,PETERSEN2018296,yarotsky18a,DBLP:journals/corr/abs-1807-00297} can only provide implicit formulas in the sense that the approximation error contains an unknown prefactor and work only for sufficiently large $N$ or $L$ larger than some unknown numbers. For example, the approximation rate in \cite{yarotsky18a} via a narrow and deep ReLU FNN is $c(d) L^{-2\alpha/d}$ with $c(d)$ unknown and for $L$ larger than a sufficiently large unknown number $\mathcal{L}$; the approximation rate in \cite{yarotsky18a} via a wide and shallow ($c_1(d)$-layer) ReLU FNN is $c_2(d) N^{-\alpha/d}$ with $c_1(d)$ and $c_2(d)$ unknown and for $N$ larger than a sufficiently large unknown number $\mathcal{N}$. For another example, given an approximation error $\epsilon$, \cite{PETERSEN2018296} proved the existence of a ReLU FNN with a constant but still unknown number of layers approximating a $C^\beta$ function within the target error. Similarly, given the $\epsilon$ error, \cite{Hadrien,bandlimit,DBLP:journals/corr/abs-1807-00297} estimate the scaling of the network size in $\epsilon$ and the scaling contains unknown prefactors. Given an arbitrary $L$ and $N$, no existing work can provide an explicit formula for the approximation error to guide practical network design, e.g., to guarantee whether the network is large enough to meet the accuracy requirement. This paper provides such formulas for the first time and in fact the bound in these formulas is asymptotically tight as we shall prove later. }

\textcolor{black}{Second, our target functions are H{\"o}lder continuous, while most of existing works aim for a smaller function space with certain smoothness, e.g. functions in $C^\alpha([0,1]^d)$ with $\alpha\geq 1$ \cite{NIPS2017_7203,DBLP:journals/corr/LiangS16,yarotsky2017,DBLP:journals/corr/abs-1807-00297}, band-limited functions \cite{bandlimit}, Korobov spaces \cite{Hadrien}, or Besev spaces \cite{suzuki2018adaptivity}. To the best of our knowledge, there is only one existing article \cite{yarotsky18a} concerning the approximation power of deep ReLU FNNs for $C([0,1]^d)$. However, the conclusion of \cite{yarotsky18a} only works for ReLU FNNs with a fixed width $2d+10$ and a sufficiently large $L$, instead of a fixed $L$ and an arbitrary $N$ as required in the nonlinear approximation (see Figure \ref{fig:existingWork} for the comparison of the conclusion of \cite{yarotsky18a} and this paper).}

\textcolor{black}{As we can see in Figure \ref{fig:existingWork}, the improvement of the best $N$-term approximation rate in terms of $N$ when $L$ is increased from $1$ to $2$ or $3$ is significant, which shows the power of depth in ReLU FNNs. However, in the case when $L>3$, our analysis shows that increasing $L$ cannot improve the approximation rate in terms of $N$. As an interesting corollary of our analysis, for any function $f$ on $[0,1]$, regardless of its smoothness and even the continuity, if $f$ can be approximated using using functions in $\mathcal{D}_1$ with the best $N$-term approximation rate $\epsilon_{L,f}=\OO(N^{-\eta})$, we show that functions in $\mathcal{D}_2$ can improve the best $N$-term approximation rate to $\epsilon_{L,f}=\OO(N^{-2\eta})$. Extending this conclusion for a general $d$ dimensional function is challenging and we leave it as future work.}

\textcolor{black}{ From the point of view of analysis techniques, this paper introduce new analysis methods merely based on the structure of FNNs, while existing works \cite{NIPS2017_7203,bandlimit,yarotsky2017,DBLP:journals/corr/LiangS16,Hadrien,suzuki2018adaptivity,PETERSEN2018296,yarotsky18a,DBLP:journals/corr/abs-1807-00297} rely on constructing FNNs to approximate traditional basis in approximation theory, e.g., polynomials, splines, and sparse grids, which are used to approximate smooth functions.  }

\begin{figure}
    \centering
    \includegraphics[width=0.6\textwidth]{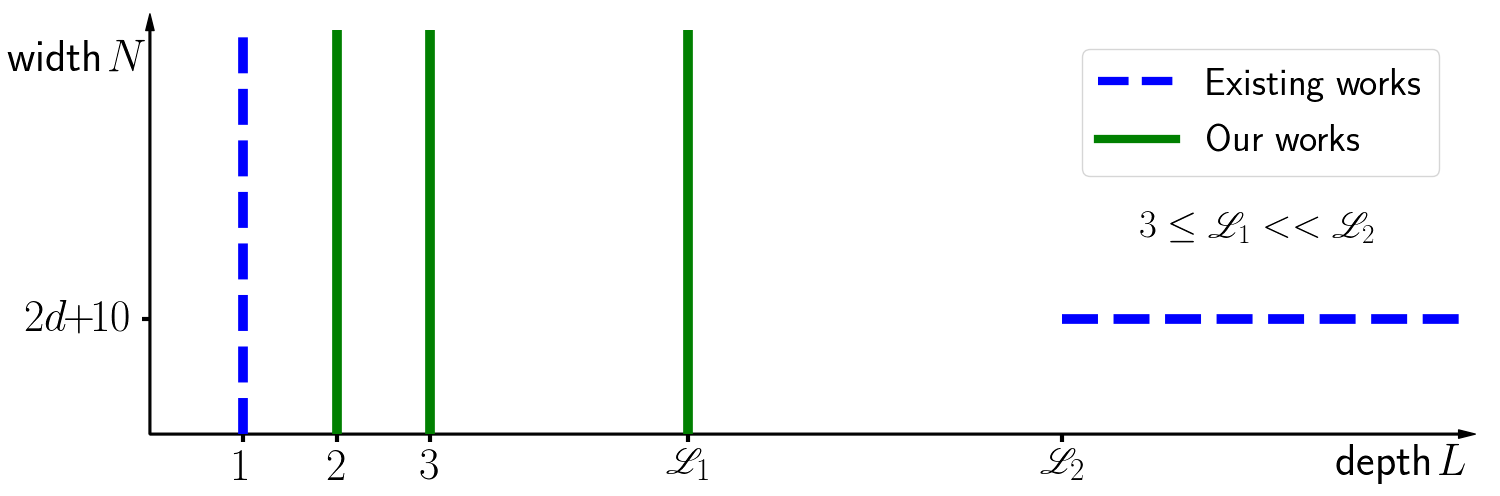}
    \caption{\textcolor{black}{A comparison of existing works and our contribution on the approximation power of ReLU FNNs for H{\"o}lder continuous functions of order $\alpha$. Existing results in two cases: 1) $\OO(N^{-\alpha/d})$ approximation rate for ReLU FNNs with depth $L=1$ and width $N$; 2) $\OO(L^{-2\alpha/d})$ approximation rate for ReLU FNNs with depth $L$ larger than a sufficiently large unknown number $\mathscr{L}_2$  and width $2d+10$. 
Our contribution: $\OO(N^{-2\alpha/d})$ approximation rate for ReLU FNNs width depth $L\geq 3$ and width $N$ in the case of $d>1$ ($L\geq 2$ in the case of $d=1$).} }
    \label{fig:existingWork}
\end{figure}

Finally, we analyze the approximation efficiency of neural networks in parallel computing, a very important point of view that was not paid attention to in the literature. In most applications, the efficiency of deep learning computation highly relies on parallel computation. We show that a narrow and very deep neural network is inefficient if its approximation rate is not exponentially better than wide and shallower networks.  Hence, neural networks with $\OO(1)$ layers are more attractive in modern computational platforms, considering the computational efficiency per training iteration in parallel computing platforms. Our conclusion does not conflict with the current state-of-the-art deep learning research since most of these successful deep neural networks have a depth that is asymptotically $\OO(1)$ relative to the width.

\subsection{Organization}

The rest of the paper is organized as follows. Section \ref{sec:notations} summarizes the notations throughout this paper. Section \ref{sec:main} presents the main theorems while Section \ref{sec:numerical} shows numerical tests in parallel computing to support the claims in this paper. Finally, Section \ref{sec:con} concludes this paper with a short discussion.

\section{Preliminaries}
\label{sec:notations}

For the purpose of convenience, we present notations and elementary lemmas used throughout this paper as follows.

\subsection{Notations}

 \begin{itemize}
     
     \item Matrices are denoted by bold uppercase letters, e.g., $\bm{A}\in\mathbb{R}^{m\times n}$ is a real matrix of size $m\times n$, and $\bm{A}^T$ denotes the transpose of $\bm{A}$. Correspondingly, $\bm{A}(i,j)$ is the $(i,j)$-th entry of $\bm{A}$; $\bm{A}(:,j)$ is the $j$-th column of $\bm{A}$; $\bm{A}(i,:)$ is the $i$-th row of $\bm{A}$.
     \item Vectors are denoted as bold lowercase letters, e.g., $\bm{v}\in \R^n$ is a column vector of size $n$ and $\bm{v}(i)$ is the $i$-th element of $\bm{v}$. $\bm{v}=[v_1,\cdots,v_n]^T=\left[\begin{array}{c}
    \vspace{-3pt} v_1 \\ \vspace{-5pt} \vdots \\ v_n
     \end{array}\right]$ are vectors consisting of numbers $\{v_i\}$ with $\bm{v}(i)=v_i$.
     
     \item  The Lebesgue measure is denoted as $\mu(\cdot)$. 
     
     \item The set difference of $A$ and $B$ is denoted by $A\backslash B$. $A^c$ denotes $[0,1]^d\backslash A$ for any $A\subseteq [0,1]^d$. 
     
     \item For a set of numbers $A$, and a number $x$, $A-x:=\{y-x:y\in A\}$.
     
     \item For any $\xi\in \R$, let $\lfloor \xi\rfloor:=\max \{i: i\le \xi,\ i\in \mathbb{Z}\}$ and $\lceil \xi\rceil:=\min \{i: i\ge \xi,\ i\in \mathbb{Z}\}$.

     \item Assume $\bm{n}\in \N^n$, then $f(\bm{n})=\mathcal{O}(g(\bm{n}))$ means that there exists positive $C$ independent of $\bm{n}$, $f$, and $g$ such that $ f(\bm{n})\le Cg(\bm{n})$ when $\bm{n}(i)$ goes to $+\infty$ for all $i$. 
     
     \item  Define $\tn{Lip}(\nu,\alpha,d)$ as the class of functions defined on $[0,1]^d$ satisfying the uniformly Lipchitz property of order $\alpha$ with a Lipchitz constant $\nu>0$.
     That is, any $f\in \tn{Lip}(\nu,\alpha,d)$ satisfies
     \begin{equation*}
         |f(\bm{x})-f(\bm{y})|\le \nu \|\bm{x}-\bm{y}\|_2^\alpha,\quad \tn{for any $\bm{x},\bm{y}\in [0,1]^d$.}
     \end{equation*}
     
     \item Let CPL($N$) be the set of continuous piecewise linear functions with $N-1$ pieces mapping $[0,1]$ to $\R$. The endpoints of each linear piece are called ``break points" in this paper.
     
     \item Let $\sigma:\R\to \R$ denote the rectified linear unit (ReLU), i.e. $\sigma(x)=\max\{0,x\}$.  With the abuse of notations, we define $\sigma:\R^d\to \R^d$ as $\sigma(\xx)=\left[\begin{array}{c}
          \max\{0,x_1\}  \\
          \vdots \\
          \max\{0,x_d\}
     \end{array}\right]$ for any $\xx=[x_1,\cdots,x_d]^T\in \R^d$.
     
     \item We will use $\tn{NN}$ as a ReLU neural network for short and use Python-type notations to specify a class of $\tn{NN}$'s, e.g., $\tn{NN}(\tn{c}_1;\ \tn{c}_2;\ \cdots;\ \tn{c}_m)$ is a set of ReLU FNN's satisfying $m$ conditions given by $\{\tn{c}_i\}_{1\leq i\leq m}$, each of which may specify the number of inputs ($\NNinput$), the total number of nodes in all hidden layers ($\#$node), the number of hidden layers ($\#$layer), the total number of parameters ($\#$parameter), and the width in each hidden layer (widthvec), the maximum width of all hidden layers (maxwidth) etc. For example, $\tn{NN}(\NNinput=2;\NNwidth=[100,100])$ is a set of NN's $\phi$ satisfying:
     \begin{itemize}
         \item $\phi$ maps from $\R^2$ to $\R$.
         \item $\phi$ has two hidden layers and the number of nodes in each hidden layer is $100$.
     \end{itemize}
     \item $[n]^L$ is short for $[n,n,\cdots,n]\in \N^L$. For example, \[\tn{NN}(\NNinput=d;\NNwidth=[100,100,100])=\tn{NN}(\NNinput=d;\NNwidth=[100]^3).\]
     \item For $\phi\in \tn{NN}(\NNinput=d;\NNwidth=[N_1,N_2,\cdots,N_L])$, if we define $N_0=d$ and $N_{L+1}=1$, then the architecture of $\phi$ can be briefly described as follows:
    \begin{equation*}
    \begin{aligned}
    \bm{x}=\tilde{\bm{h}}_0 \myto^{\bm{W}_1,\ \bm{b}_1} \bm{h}_1\mathop{\longrightarrow}^{\sigma} \tilde{\bm{h}}_1 \ \cdots \
    \myto^{\bm{W}_L,\ \bm{b}_L} \bm{h}_L\mathop{\longrightarrow}^{\sigma} \tilde{\bm{h}}_L \mathop{\myto}^{\bm{W}_{L+1},\ \bm{b}_{L+1}} \phi(\bm{x})=\bm{h}_{L+1},
    \end{aligned}
    \end{equation*}
    where $\bm{W}_i\in \R^{N_{i}\times N_{i-1}}$ and $\bm{b}_i\in \R^{N_i}$ are the weight matrix and the bias vector in the $i$-th linear transform in $\phi$, respectively, i.e., $\bm{h}_i :=\bm{W}_i \tilde{\bm{h}}_{i-1} + \bm{b}_i$ for $\ i=1,2,\cdots,L+1$ and $\tilde{\bm{h}}_i=\sigma(\bm{h}_i)$ for $i=1,2,\cdots,L$.
\end{itemize}
     
\subsection{Lemmas}
\label{sub:lem}

Let us study the properties of ReLU FNNs with only one hidden layer to warm up in Lemma \ref{OneLayerProperty} below. It indicates that $\tn{CPL}(N+1)=\tn{NN}(\NNinput=1;\NNwidth=[N])$ for any $N\in \N^+$. 

\begin{lemma}
    \label{OneLayerProperty}
    Suppose $\phi\in \tn{NN}(\NNinput=1;\NNwidth=[N])$ with an architecture:
    \begin{equation*}
    \begin{aligned}
    x \myto^{\bm{W}_1,\ \bm{b}_1} \bm{h}\mathop{\myto}^{\sigma} \tilde{\bm{h}} \mathop{\myto}^{\bm{W}_2,\ \bm{b}_2} \phi(x).
    \end{aligned}
    \end{equation*}
Then $\phi$ is a continuous piecewise linear function. Let $\bm{W}_1=[1,1,\cdots,1]^T\in\R^{  N\times 1}$, then we have:
    \begin{enumerate}[(1)]
        \item Given a sequence of strictly increasing numbers $x_0$, $x_1$,$\cdots$, $x_N$, there exists $\bm{b}_1\in\R^{  N}$ independent of $\bm{W}_2$ and $\bm{b}_2$ such that the break points of $\phi$ are exactly $x_0$, $\cdots$, $x_N$ on the interval $[x_0,x_N]$\footnote{We only consider the interval $[x_0,x_N]$ and hence $x_0$ and $x_N$ are treated as break points. $\phi(x)$ might not have a real break point in a small open neighborhood of $x_0$ or $x_N$.}.
        \item Suppose $\{x_i\}_{i\in\{0,1,\cdots,N\}}$ and $\bm{b}_1$ are given in (1). 
        Given any sequence $\{y_i\}_{i\in\{0,1,\cdots,N\}}$, there exist $\bm{W}_2$ and $\bm{b}_2$ such that $\phi(x_i)=y_i$ for $i=0,1,\cdots,N$ and $\phi$ is linear on $[x_i,x_{i+1}]$ for $i=0,1,\cdots,N-1$. 
    \end{enumerate}
\end{lemma}

Part $(1)$ in Lemma \ref{OneLayerProperty} follows by setting $\bm{b}_1=[-x_0,-x_1,\cdots,-x_{N-1}]^T$. The existence in Part $(2)$ is equivalent to the existence of a solution of linear equations, which is left for the reader. Next, we study the properties of ReLU FNNs with two hidden layers. In fact, we can show that the closure of  ${\tn{NN}}(\NNinput=1;\NNwidth=[2m,2n+1])$ contains $ \tn{CPL}(mn+1)$ for any $m,n\in \N^+$, where the closure is in the sense of ${L^p}$-norm for any $p\in [1,\infty)$. The proof of this property relies on the following lemma.
\begin{lemma}
    \label{SquarePointsLemma}
    For any $m,n\in \N^+$, given any $m(n+1)+1$ samples $(x_i,y_i)\in \R^2$ with $x_0<x_1<x_2<\cdots<x_{m(n+1)}$ and $y_i\ge 0$ for $i=0,1,\cdots,m(n+1)$, there exists $\phi\in \tn{NN}(\NNinput=1;\NNwidth=[2m,2n+1])$ satisfying the following conditions:
    \begin{enumerate}[(1)]
        \item $\phi(x_i)=y_i$ for $i=0,1,\cdots,m(n+1)$;
        \item $\phi$ is linear on each interval $[x_{i-1},x_{i}]$ for $i\notin \{(n+1)j:j=1,2,\cdots,m\}$;
        \item $\displaystyle\sup_{x\in[x_0,\,x_{m(n+1)}]} |\phi(x)| \le 3\max_{i\in \{0,1,\cdots,m(n+1)\}}y_i \prod_{k=1}^{n}\left(1+\tfrac{\max\{x_{j(n+1)+n}-x_{j(n+1)+k-1}:j=0,1,\cdots,m-1\} }
        {\min\{x_{j(n+1)+k}-x_{j(n+1)+k-1}:j=0,1,\cdots,m-1\} }\right)$.
    \end{enumerate}
\end{lemma}
\begin{proof}
    For simplicity of notation, we define $I_0(m,n)\coloneqq\{0,1,\cdots,m(n+1)\}$, $I_1(m,n)\coloneqq\{j(n+1):j=1,2,\cdots,m\}$, and $I_2(m,n)\coloneqq I_0(m,n)\backslash I_1(m,n)$ for any $m,n\in \N^+$.
    Since $\phi\in \tn{NN}(\NNinput=1;\NNwidth=[2m,2n+1])$, the  architecture of $\phi$ is
\begin{equation}
\label{eqn:arch}
\begin{aligned}
x \myto^{\bm{W}_1,\ \bm{b}_1} \bm{h}\mathop{\myto}^{\sigma} \tilde{\bm{h}} \myto^{\bm{W}_2,\ \bm{b}_2} \bm{g}\mathop{\myto}^{\sigma} \tilde{\bm{g}}\mathop{\myto}^{\bm{W}_3,\ \bm{b}_3} \phi(x).
\end{aligned}
\end{equation}
Note that $\bm{g}$ maps $x\in\R$ to $\bm{g}(x)\in\R^{2n+1}$ and hence each entry of $\bm{g}(x)$ itself is a sub-network with one hidden layer. Denote $\bm{g}=[g_0,g_1^+,g_1^-,\cdots,g_{n}^+,g_{n}^-]^T$, then $\{g_0,g_1^+,g_1^-,\cdots,g_{n}^+ ,g_{n}^-\}\subseteq \tn{NN}(\NNinput=1;\NNwidth=[2m])$. Our proof of Lemma \ref{SquarePointsLemma} is mainly based on the repeated applications of Lemma \ref{OneLayerProperty} to determine parameters of $\phi(x)$ such that Conditions $(1)$ to $(3)$ hold.

\beforestepskip
\noindent\textbf{Step $1$}: Determine $\bm{W}_1$ and $\bm{b}_1$. 
\afterstepskip

By Lemma \ref{OneLayerProperty}, $\exists\ \bm{W}_1=[1,1,\cdots,1]^T\in \R^{ 2m\times 1}$ and $\bm{b}_1\in \R^{ 2m\times 1}$ such that sub-networks in $\{g_0,g_1^+,g_1^-,\cdots,g_{n}^+ ,g_{n}^-\}$ have the same set of break points: $\Big\{x_i: i\in I_1(m,n)\cup \big(I_1(m,n)-1\big) \cup \{0\}\Big\}$, no matter what $\bm{W}_2$ and $\bm{b}_2$ are. 

\beforestepskip
\noindent\textbf{Step $2$}:  Determine $\bm{W}_2$ and $\bm{b}_2$. 
\afterstepskip

This is the key step of the proof. Our ultimate goal is to set up $\bm{g}=[g_0,g_1^+,g_1^-,\cdots,g_{n}^+,g_{n}^-]^T$ such that, after a nonlinear activation function, there exists a linear combination in the last step of our network (specified by $\bm{W}_3$ and $\bm{b}_3$ as shown in \eqref{eqn:arch}) that can generate a desired $\phi(x)$ matching the sample points $\{(x_i,y_i)\}_{0\leq i\leq m(n+1)}$. In the previous step, we have determined the break points of $\{g_0,g_1^+,g_1^-,\cdots,g_{n}^+ ,g_{n}^-\}$ by setting up $\bm{W}_1$ and $\bm{b}_1$; in this step, we will identify $\bm{W}_2\in \R^{(2n+1)\times 2m}$ and $\bm{b}_2\in \R^{2n+1}$ to fully determine $\{g_0,g_1^+,g_1^-,\cdots,g_{n}^+ ,g_{n}^-\}$. This will be conducted in two sub-steps.

\beforestepskip
\noindent\textbf{Step $2.1$}: Set up.
\afterstepskip

Suppose $f_0(x)$ is a continuous piecewise linear function defined on $[0,1]$ fitting the given samples $f_0(x_i)=y_i$ for $i\in I_0(m,n)$, and $f_0$ is linear between any two adjacent points of $\{x_i:i\in I_0(m,n)\}$. We are able to choose $\bm{W}_2(1,:)$ and $\bm{b}_2(1)$ such that $g_0(x_i)=f_0(x_i)$ for $i\in I_1(m,n)\cup (I_1(m,n)-1)\cup \{0\}$ by Lemma \ref{OneLayerProperty}, since there are $2m+1$ points in $I_1(m,n)\cup (I_1(m,n)-1)\cup \{0\}$. Define $f_1:=f_0-\tilde{g}_0$, where $\tilde{g}_0=\sigma (g_0)=g_0$ as shown in Equation \eqref{eqn:arch}, since $g_0$ is positive by the construction of Lemma \ref{OneLayerProperty}. Then we have $f_1 (x_i)=f_0 (x_i)-\tilde{g}_0 (x_i)=0$ for $i\in  (I_1 (m,n)-n-1) \cup \{m (n+1)\}$. See Figure \ref{fig:lm} (a) for an illustration of $f_0$, $f_1$, and $g_0$.

\beforestepskip
\noindent\textbf{Step $2.2$}: Mathematical induction.
\afterstepskip

For each $k\in \{1,2,\cdots,n\}$, given $f_k$, we determine $\bm{W}_2 (2k,:)$, $\bm{b}_2 (2k)$, $\bm{W}_2 (2k+1,:)$, and $\bm{b}_2 (2k+1)$, to completely specify 
$g^+_k$ and $g^-_k$, which in turn can determine $f_{k+1}$. Hence, it is only enough to show how to proceed with an arbitrary $k$, since the initialization of the induction, $f_1$, has been constructed in Step $2.1$. See Figure \ref{fig:lm} (b)-(d) for the illustration of the first two induction steps. We recursively rely on the fact of $f_k$ that
\begin{itemize}
    \item $f_k (x_i)=0$ for $i\in \cup_{\ell=0}^{k-1} (I_1 (m,n)-n-1+\ell)\cup\{m (n+1)\}$,
    \item ${f}_{k}$ is linear on each interval $[x_{i-1},x_{i}]$ for $i\in I_2 (m,n)\backslash \{0\}$,
\end{itemize}
\vspace{0.2cm}
to construct $f_{k+1}$ satisfying similar conditions as follows: 
\begin{itemize}
    \item $f_{k+1} (x_i)=0$ for $i\in \cup_{\ell=0}^{k} (I_1 (m,n)-n-1+\ell)\cup\{m (n+1)\}$,
    \item ${f}_{k+1}$ is linear on each interval $[x_{i-1},x_{i}]$ for $i\in I_2 (m,n)\backslash \{0\}$.
\end{itemize}
\vspace{0.2cm}
The induction process for $\bm{W}_2 (2k,:)$, $\bm{b}_2 (2k)$, $\bm{W}_2 (2k+1,:)$, $\bm{b}_2 (2k+1)$, and $f_{k+1}$ can be divided into four parts.

\beforestepskip
\noindent\textbf{Step $2.2.1$}: Define index sets.
\afterstepskip

Let $\Lambda^+_k=\{j: f_k (x_{j (n+1)+k})\ge 0,\  0\le j<m\}$ and $\Lambda^-_k=\{j: f_k (x_{j (n+1)+k})< 0,\ 0\le j<m\}$. The cardinality of $\Lambda^+_k\cup \Lambda^-_k$ is $m$. We will use $\Lambda^+_k$ and $\Lambda^-_k$ to generate $2m+1$ samples to determine CPL functions $g^+_{k}(x)$ and $g^-_{k}(x)$ in the next step.
\begin{figure}[htb]
    \twofig[0.8]{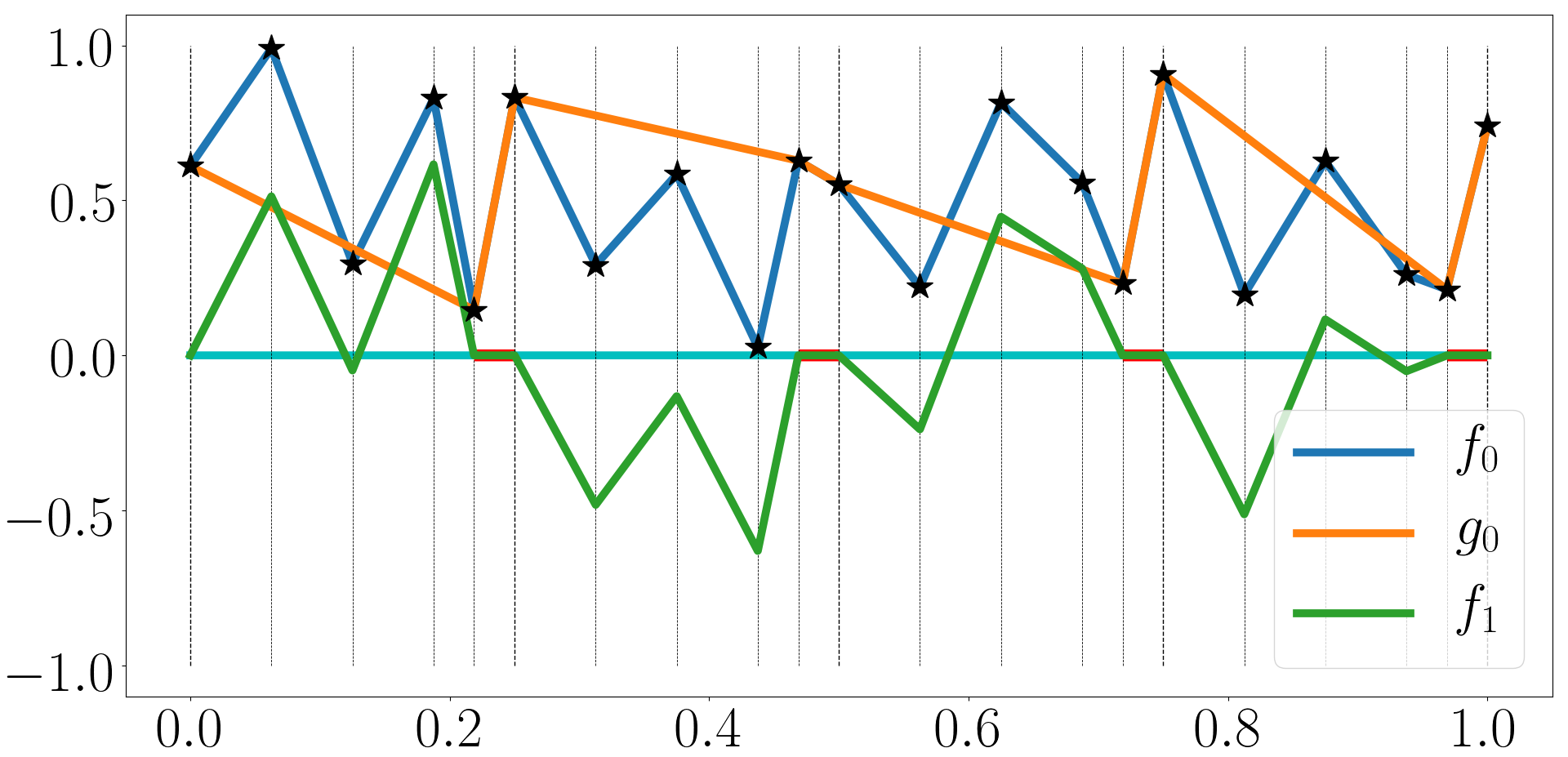}{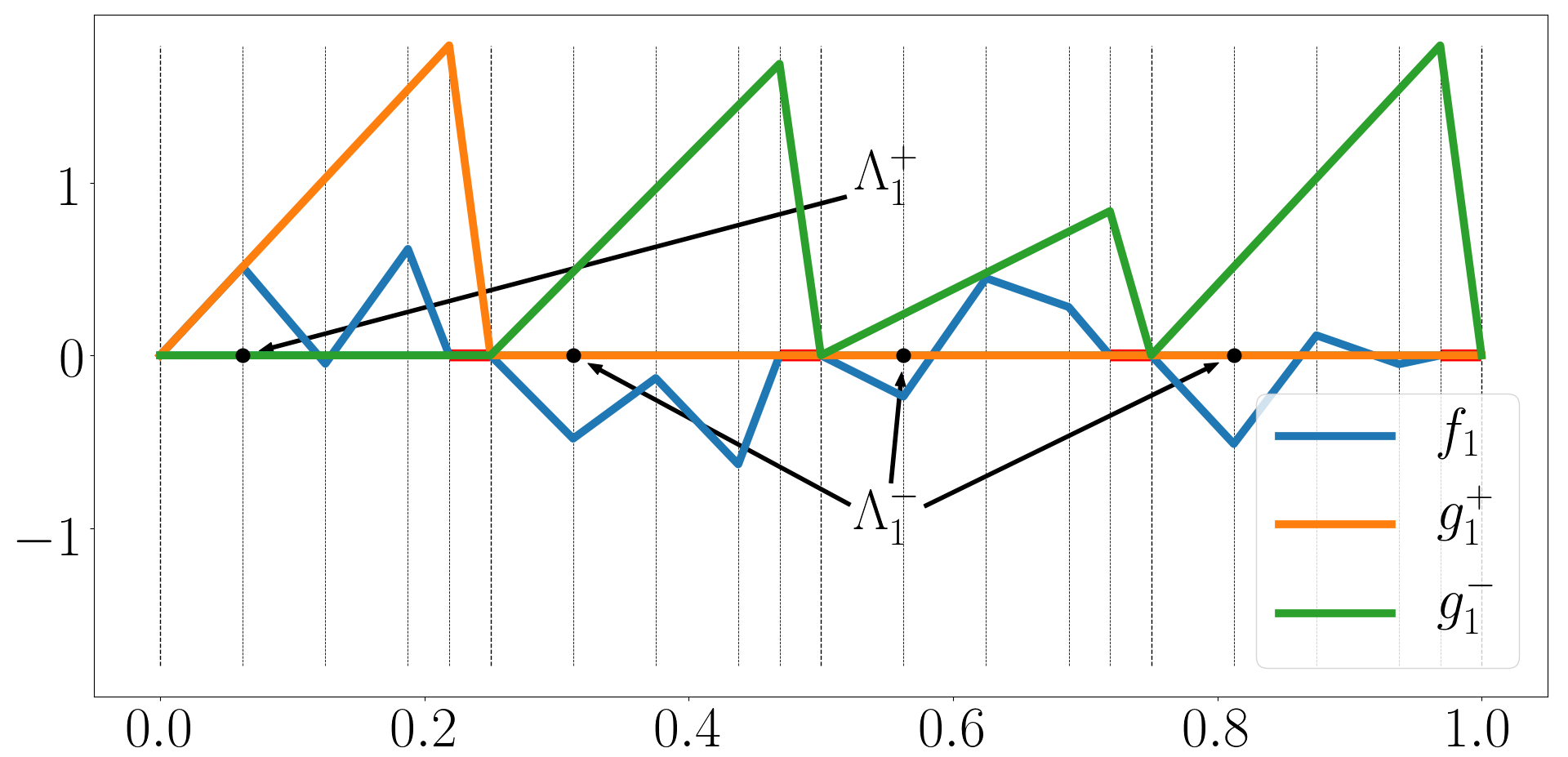}            
    \twofig[0.8]{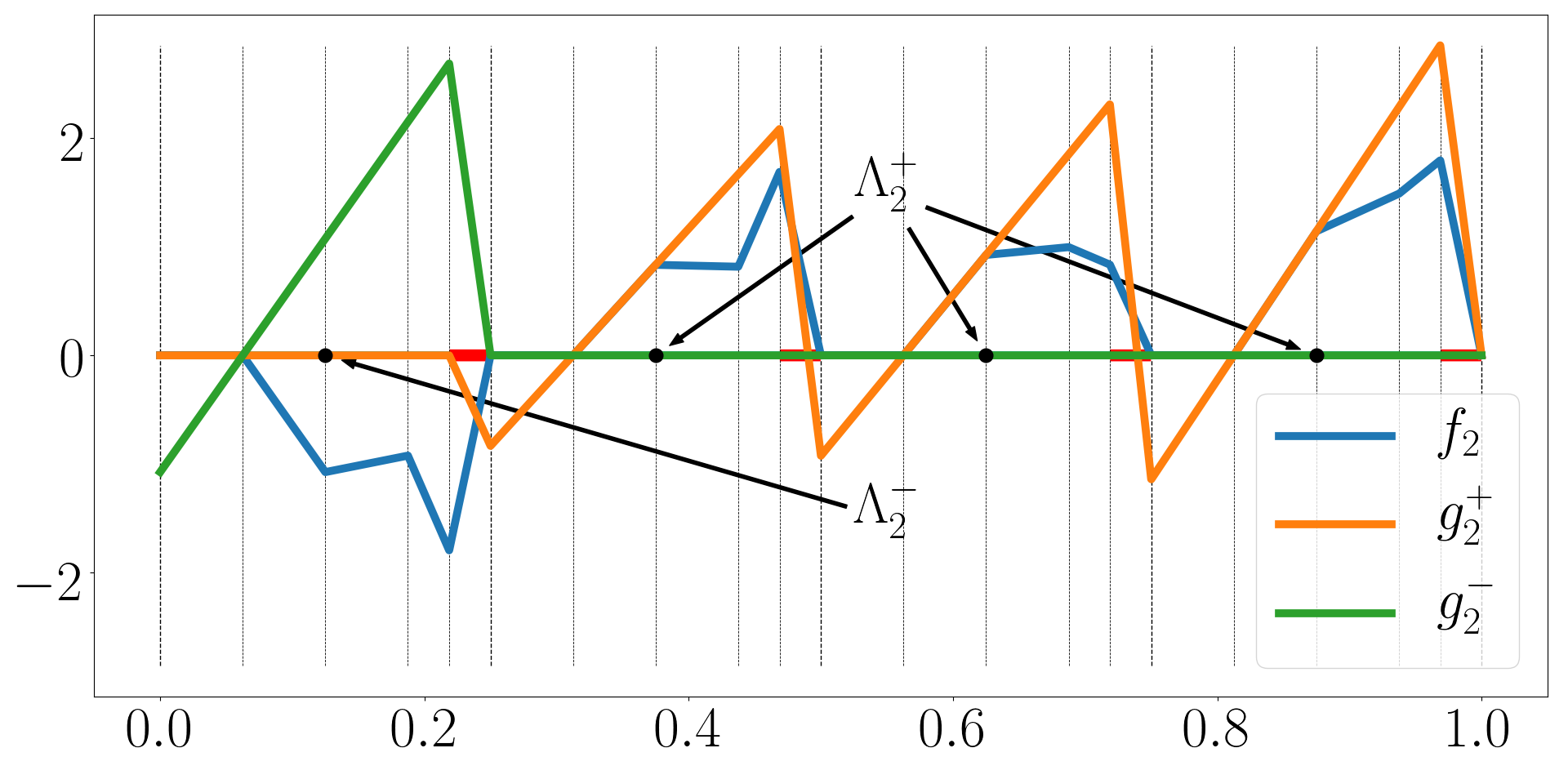}{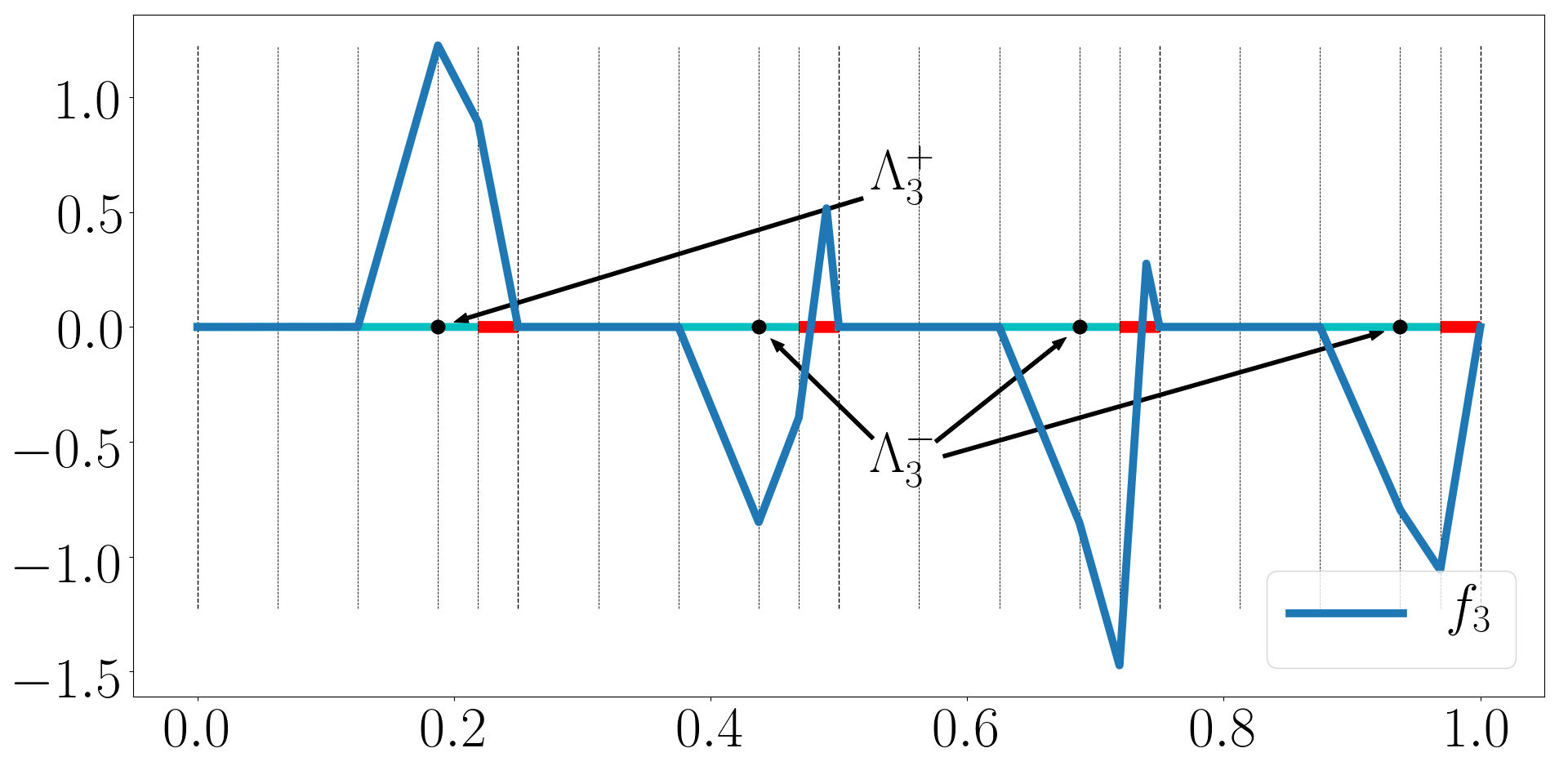}
        \caption{ Illustrations of the proof of Lemma \ref{SquarePointsLemma}, especially Step $2$ of the proof, when $m=n=4$, with the ``don't-care" region in red.  (a) Given samples $\{(x_i,y_i):i=0,1,\cdots,m(n+1)\}$ marked with ``star'' signs, suppose $f_0(x)$ is a CPL function fitting the samples, construct $g_0$ such that $f_1=f_0-\sigma(g_0)$ is closer to $0$ than $f_0$ in the $L^\infty$ sense. (b) Construct $g^+_1$ and $g^-_1$ such that $f_2=f_1-\sigma(g^+_1)+\sigma(g^-_1)$ is closer to $0$ than $f_1$ in the $L^\infty$ sense in a subset of the ``important" region. (c) Construct $g^+_2$ and $g^-_2$ such that $f_3=f_2-\sigma(g^+_2)+\sigma(g^-_2)$ is closer to $0$ than $f_2$ in the $L^\infty$ sense in a larger subset of the ``important" region. (d) The visualization of $f_3$, which is $0$ in the ``important" areas that have been processed and may remain large near the ``don't-care" region. $f_k$ will decay quickly outside the ``don't-care" region as $k$ increases.}
        \label{fig:lm}
\end{figure}

\beforestepskip
\noindent\textbf{Step $2.2.2$}:  Determine $\bm{W}_2 (2k,:)$ and $ \bm{b}_2 (2k)$. 
\afterstepskip

By Lemma \ref{OneLayerProperty}, we can choose $\bm{W}_2 (2k,:)$ and $\bm{b}_2 (2k)$ to fully determine $g^+_{k}(x)$ such that each $g^+_{k} (x_i)$ matches a specific value for $i\in  (I_1 (m,n)-n-1)\cup  (I_1 (m,n)-1)\cup \{m (n+1)\}$. The values of $\big\{g^+_{k} (x_i):i\in(I_1 (m,n)-n-1)\cup  (I_1 (m,n)-1)\cup \{m (n+1)\}\big\}$ are specified as:   
\begin{itemize}
    \item If $j\in \Lambda^+_k$, specify the values of $g^+_{k} (x_{j (n+1)})$ and $g^+_{k} (x_{j (n+1)+n})$ such that
    $g^+_{k}(x_{j (n+1)+k-1})=0$ and $g^+_{k}(x_{j (n+1)+k})=f_k (x_{j (n+1)+k})$. The existence of these values fulfilling the requirements above comes from the fact that $g^+_{k} (x)$ is linear on the interval $[x_{j (n+1)},x_{j (n+1)+n}]$ and $g^+_{k} (x)$ only depends on the values of $g^+_{k}(x_{j (n+1)+k-1})$ and $g^+_{k}(x_{j (n+1)+k})$ on $[x_{j (n+1)},x_{j (n+1)+n}]$. Now it is easy to verify that $\tilde{g}^+_{k}(x):=\sigma (g^+_{k}(x))$ satisfies $\tilde{g}^+_{k} (x_{j (n+1)+k})= f_k (x_{j (n+1)+k})\ge 0$ and $\tilde{g}^+_{k} (x_{j (n+1)+\ell}) = 0$ for $\ell=0,1,\cdots,k-1$, and $\tilde{g}^+_{k}$ is linear on each interval $[x_{j (n+1)+\ell},x_{j (n+1)+\ell+1}]$ for $\ell=0,1,\cdots,n-1$.
    \item If $j\in \Lambda^-_k$, let $g^+_{k} (x_{j (n+1)})=g^+_{k} (x_{j (n+1)+n})=0$. Then $\tilde{g}^+_k(x)=0$ on the interval $[x_{j (n+1)},x_{j (n+1)+n}]$.
    \item Finally, specify the value of $g^+_{k}(x)$ at $x=x_{m(n+1)}$ as $0$.
\end{itemize}

\beforestepskip
\noindent\textbf{Step $2.2.3$}: Determine $\bm{W}_2 (2k+1,:)$ and $\bm{b}_2 (2k+1)$. 
\afterstepskip

Similarly, we choose $\bm{W}_2 (2k+1,:)$ and $\bm{b}_2 (2k+1)$ such that $g^-_{k}(x)$ matches specific values as follows:
\begin{itemize}
    \item If $j\in \Lambda^-_k$, specify the values of $g^-_{k} (x_{j (n+1)})$ and $g^-_{k} (x_{j (n+1)+n})$ such that 
    $g^-_{k}(x_{j (n+1)+k-1})=0$  and $g^-_{k}(x_{j (n+1)+k})=-f_k (x_{j (n+1)+k})$. Then $\tilde{g}^-_{k}(x):=\sigma (g^-_{k}(x))$ satisfies $\tilde{g}^-_{k} (x_{j (n+1)+k})=-f_k (x_{j (n+1)+k})> 0$ and $\tilde{g}^-_{k} (x_{j (n+1)+\ell})=0$ for $\ell=0,1,\cdots,k-1$, and $\tilde{g}^-_{k}(x)$ is linear on each interval $[x_{j (n+1)+\ell},x_{j (n+1)+\ell+1}]$ for $\ell=0,1,\cdots,n-1$.
    \item If $j\in \Lambda^+_k$, let $g^-_{k} (x_{j (n+1)})=g^-_{k} (x_{j (n+)+n})=0$. Then $\tilde{g}^-_k(x)=0$ on the interval $[x_{j (n+1)},x_{j (n+1)+n}]$.
    \item Finally, specify the value of $g^-_{k}(x)$ at $x=x_{m(n+1)}$ as $0$.
\end{itemize} 

\beforestepskip
\noindent\textbf{Step $2.2.4$}: Construct $f_{k+1}$ from $g^+_k$ and $g^-_k$.
\afterstepskip

For the sake of clarity,  the properties of $g^+_{k}$ and $g^-_{k}$ constructed in Step $2.2.3$ are summarized below:
\begin{enumerate}[(1)]
    \item $ f_k (x_i)=\tilde{g}_{k}^+ (x_i)=\tilde{g}_{k}^- (x_i)=0$ for  $i\in \cup_{\ell=0}^{k-1} (I_1 (m,n)-n-1+\ell)\cup\{m (n+1)\}$;
    
    \item If $j\in \Lambda^+_k$,\  $\tilde{g}_{k}^+ (x_{j (n+1)+k})=f_k (x_{j (n+1)+k})\ge 0$  and $\tilde{g}_{k}^- (x_{j (n+1)+k})=0$;

    \item If $j\in\Lambda^-_k$,\  $\tilde{g}_{k}^- (x_{j (n+1)+k})=-f_k (x_{j (n+1)+k})>0$  and $\tilde{g}_{k}^+ (x_{j (n+1)+k})=0$;
    
    \item $\tilde{g}_{k}^+$ and $\tilde{g}_{k}^-$ are linear on each interval $[x_{j (n+1)+\ell},x_{j (n+1)+\ell+1}]$ for $\ell=0,1,\cdots,n-1,\ j\in \Lambda^+_k\cup \Lambda^-_k=\{0,1,\cdots,m-1\}$. In other words, $\tilde{g}_{k}^+$ and $\tilde{g}_{k}^-$ are linear on each interval $[x_{i-1},x_{i}]$ for $i\in I_2 (m,n)\backslash \{0\}$. 
\end{enumerate} 
See Figure \ref{fig:lm} (a)-(c) for the illustration of $g_0$, $g^+_1$, $g^-_1$, $g^+_2$, and $g^-_2$, and to verify their properties as listed above.

Note that $\Lambda^+_k\cup \Lambda^-_k=\{0,1,\cdots,m-1\}$, so $f_k (x_i)-\tilde{g}_{k}^+ (x_i)+\tilde{g}_{k}^- (x_{i})=0$ for $i\in \cup_{\ell=0}^{k} (I_1 (m,n)-n-1+\ell)\cup\{m (n+1)\}$.
Now we define $f_{k+1}:=f_k-\tilde{g}^+_{k}+\tilde{g}^-_{k}$, then
\begin{itemize}
    \item $f_{k+1} (x_i)=0$ for $i\in \cup_{\ell=0}^{k} (I_1 (m,n)-n-1+\ell)\cup\{m (n+1)\}$;
    \item $f_{k+1}$ is linear on each interval $[x_{i-1},x_{i}]$ for $i\in I_2 (m,n)\backslash \{0\}$.
\end{itemize}
See Figure \ref{fig:lm} (b)-(d) for the illustration of $f_1$, $f_2$, and $f_3$, and to verify their properties as listed just above. This finishes the mathematical induction process. As we can imagine based on Figure \ref{fig:lm}, when $k$ increases, the support of $f_k$ shrinks to the ``don't-care" region. 

\beforestepskip
\noindent\textbf{Step $3$}: Determine $\bm{W}_3$ and $\bm{b}_3$.
\afterstepskip

With the special vector function $\bm{g}=[g_0,g_1^+,g_1^-,\cdots,g_{n}^+,g_{n}^-]^T$ constructed in Step $2$, we are able to specify $\bm{W}_3$ and $\bm{b}_3$ to generate a desired $\phi(x)$ with a well-controlled $L^\infty$-norm matching the samples $\{(x_i,y_i)\}_{0\leq i\leq m(n+1)}$.

In fact, we can simply set $\bm{W}_3=[1,1,-1,1,-1,\cdots,1,-1]\in \R^{1\times (2n+1)}$ and $\bm{b}_3=0$, which finishes the construction of $\phi(x)$. The rest of the proof is to verify the properties of $\phi(x)$. Note that $\phi=\tilde{g}_0+\sum_{\ell=1}^{n}\tilde{g}^+_\ell-\sum_{\ell=1}^{n}\tilde{g}^-_\ell$. By the mathematical induction, we have:
\begin{itemize}
    \item $f_{n+1}=f_0-\tilde{g}_0-\sum_{\ell=1}^{n}\tilde{g}^+_\ell+\sum_{\ell=1}^{n}\tilde{g}^-_\ell$;
    
    \item $f_{n+1} (x_i)=0$ for $i\in \cup_{\ell=0}^{n} (I_1 (m,n)-n-1+\ell)\cup\{m (n+1)\}=I_0 (m,n)$;
    
    \item ${f}_{n+1}$ is linear on each interval $[x_{i-1},x_{i}]$ for $i\in I_2 (m,n)\backslash \{0\}$.
\end{itemize}
\vspace{0.25cm}
Hence, $\phi=\tilde{g}_0+\sum_{\ell=1}^{n}\tilde{g}^+_\ell-\sum_{\ell=1}^{n}\tilde{g}^-_\ell=f_0-f_{n+1}$. Then $\phi$ satisfies Conditions $(1)$ and $(2)$ of this lemma. It remains to check that $\phi$ satisfies Condition $(3)$. 

By the definition of $f_1$, we have 
\begin{equation}
\sup_{x\in[x_0,\,x_{m(n+1)}]}|f_1 (x)| \le 2\max\{y_i:i\in I_0 (m,n)\}. 
\end{equation}
By the induction process in Step $2$, for $k\in\{1,2,\cdots,n\}$, it holds that
\begin{equation}
\sup_{x\in[x_0,\,x_{m(n+1)}]}|\tilde{g}^+_k(x)|\le \tfrac{\max\{x_{j(n+1)+n}-x_{j(n+1)+k-1}:j=0,1,\cdots,m-1\} }
{\min\{x_{j(n+1)+k}-x_{j(n+1)+k-1}:j=0,1,\cdots,m-1\} } \sup_{x\in[x_0,\,x_{m(n+1)}]}|f_k(x)|
\end{equation}
and 
\begin{equation}
\sup_{x\in[x_0,\,x_{m(n+1)}]}|\tilde{g}^-_k(x)|\le \tfrac{\max\{x_{j(n+1)+n}-x_{j(n+1)+k-1}:j=0,1,\cdots,m-1\} }
{\min\{x_{j(n+1)+k}-x_{j(n+1)+k-1}:j=0,1,\cdots,m-1\} } \sup_{x\in[x_0,\,x_{m(n+1)}]}|f_k(x)|.
\end{equation}
Since either $\tilde{g}^+_k(x)$ or $\tilde{g}^-_k(x)$ is equal to $0$ on $[0,1]$, we have 
\begin{equation}
\sup_{x\in[x_0,\,x_{m(n+1)}]}|\tilde{g}^+_k(x)-\tilde{g}^-_k(x)|\le \tfrac{\max\{x_{j(n+1)+n}-x_{j(n+1)+k-1}:j=0,1,\cdots,m-1\} }
{\min\{x_{j(n+1)+k}-x_{j(n+1)+k-1}:j=0,1,\cdots,m-1\} } \sup_{x\in[x_0,\,x_{m(n+1)}]}|f_k(x)|.
\end{equation}
Note that $f_{k+1}=f_k-\tilde{g}^+_k+\tilde{g}^-_k$, which means 
\begin{equation}
\begin{split}
\sup_{x\in[x_0,\,x_{m(n+1)}]}|f_{k+1}(x)|&\le 
\sup_{x\in[x_0,\,x_{m(n+1)}]}|\tilde{g}^+_k(x)-\tilde{g}^-_k(x)|+\sup_{x\in[x_0,\,x_{m(n+1)}]}|f_{k}(x)|\\
&\le
\left(\tfrac{\max\{x_{j(n+1)+n}-x_{j(n+1)+k-1}:j=0,1,\cdots,m-1\} }
{\min\{x_{j(n+1)+k}-x_{j(n+1)+k-1}:j=0,1,\cdots,m-1\} }+1\right) \sup_{x\in[x_0,\,x_{m(n+1)}]}|f_k(x)|
\end{split}
\end{equation}
for $k\in \{1,2,\cdots,n\}$. Then we have
\begin{equation*}
\begin{split}
\sup_{x\in[x_0,\,x_{m(n+1)}]}|f_{n+1}(x)|
\le 2\max_{i\in I_0 (m,n)} y_i\prod_{k=1}^{n}\left(\tfrac{\max\{x_{j(n+1)+n}-x_{j(n+1)+k-1}:j=0,1,\cdots,m-1\} }
{\min\{x_{j(n+1)+k}-x_{j(n+1)+k-1}:j=0,1,\cdots,m-1\} }+1\right) .
\end{split}
\end{equation*}
Hence 
\begin{equation*}
\begin{split}
\sup_{x\in[x_0,\,x_{m(n+1)}]}|\phi(x)|
&\le \sup_{x\in[x_0,\,x_{m(n+1)}]}|f_0(x)|+\sup_{x\in[x_0,\,x_{m(n+1)}]}|f_{n+1}(x)|\\
&\le 3\max_{i\in I_0 (m,n)} y_i\prod_{k=1}^{n}\left(\tfrac{\max\{x_{j(n+1)+n}-x_{j(n+1)+k-1}:j=0,1,\cdots,m-1\} }
{\min\{x_{j(n+1)+k}-x_{j(n+1)+k-1}:j=0,1,\cdots,m-1\} }+1\right) .
\end{split}
\end{equation*}
So, we finish the proof.
\end{proof}

\section{Main Results}
\label{sec:main}

We present our main results in this section. First, we quantitatively prove an achievable approximation rate in the $N$-term nonlinear approximation by construction, i.e., the lower bound of the approximation rate. Second, we show a lower bound of the approximation rate asymptotically, i.e., no approximant exists asymptotically following the approximation rate. Finally, we discuss the efficiency of the nonlinear approximation considering the approximation rate and parallel computing in FNNs together. 

\subsection{Quantitative Achievable Approximation Rate}
\label{sub:upp}
    
\begin{theorem}
    \label{thm:up}
    For any $N\in \N^+$ and $f\in \tn{Lip}(\nu,\alpha,d)$ with $\alpha\in(0,1]$, we have\footnote{It is easy to generalize the results in Theorem \ref{thm:up} and Corollary \ref{coll:upp} from $L^1$ to $L^p$-norm for $p\in[1,\infty)$ since $\mu([0,1]^d)=1$.}:
    \begin{enumerate}[(1)]
    \item  If $d=1$, $\exists\ \phi\in \tn{NN}(\NNinput=1;\NNwidth=[2N,2N+1])$ such that
    \[
    \|\phi-f\|_{L^1([0,1])}\le 2 \nu N^{-2\alpha},\quad \tn{ for any }\ N\in \N^+;
    \]
    \item If $d>1$, $\exists\ \phi\in \tn{NN}(\NNinput=d;\NNwidth=[2d\lfloor N^{2/d}\rfloor,2N+2,2N+3])$ such that
    \[\|\phi-f\|_{L^1([0,1]^d)}\le 2(2\sqrt{d})^\alpha \nu N^{-2\alpha/d},\quad \tn{ for any }\ N\in \N^+.\]
    \end{enumerate}
\end{theorem}

\begin{proof}
Without loss of generality, we assume $f(\bm{0})=0$ and $\nu=1$. 

\beforestepskip
\noindent\textbf{Step $1$}: The case $d=1$.
\afterstepskip

Given any $f\in \tn{Lip} (\nu,\alpha,d)$ and $N\in \N^+$, we know $ |f (x)|\le 1 $ for any $x\in [0,1]$ since $f(0)=0$ and $\nu=1$. Set $\bar{f}=f+1\geq 0$, then $0\le \bar{f}(x)\le 2$ for any $x\in [0,1]$. Let $X=\{\frac{i}{N^2}:i=0,1,\cdots,N^2\}\cup\{\frac{i}{N}-\delta:i=1,2,\cdots,N\}$, where $\delta$ is a sufficiently small positive number depending on $N$, and satisfying \eqref{OneDDeltaCond1}. Let us order $X$ as $x_0< x_1<\cdots<x_{N (N+1)}$. By Lemma \ref{SquarePointsLemma}, given the set of samples $\big\{(x_i,\bar{f} (x_i)):i\in\{0,1,\cdots,N (N+1)\}\big\}$, there exists $\phi\in \tn{NN} (\NNinput=1;\NNwidth=[2N,2N+1])$ such that
\begin{itemize}
    \item $\phi (x_i)=\bar{f} (x_i)$ for $i=0,1,\cdots,N(N+1)$;
    \item $\phi$ is linear on each interval $[x_{i-1},x_{i}]$ for $i\notin \{(N+1)j:j=1,2,\cdots,N\}$;
    \item $\phi$ has an upper bound estimation: $\sup \{\phi (x):x\in[0,1]\} \le 6(N+1)!$.
\end{itemize} 
It follows that
\begin{equation*}
    |\bar{f}(x)-\phi(x)|\le (x_i-x_{i-1})^\alpha\le N^{-2\alpha}, \quad \tn{if } x\in [x_{i-1},x_i],\quad \tn{for }
    i\notin \{(N+1)j:j=1,2,\cdots,N\}.
\end{equation*}
Define $H_0=\cup_{i\in \{(N+1)j:j=1,2,\cdots,N\}  } [x_{i-1},x_i]$,  then
\begin{equation}
\label{UpperBoundExceptBadRegionOneD}
|\bar{f} (x)-\phi (x)|\le N^{-2\alpha}, \quad\tn{  for any } x\in [0,1]\backslash H_0,
\end{equation}    
by the fact that $\bar{f}\in \tn{Lip}(1,\alpha,d)$ and points in $X$ are equispaced. By  $\mu(H_0)\le N\delta$, it follows that 
\begin{equation*}
\begin{split}
\|\bar{f}-\phi\|_{L^1([0,1])}&=\int_{H_0}|\bar{f} (x)-\phi (x)|dx+\int_{[0,1]\backslash H_0}|\bar{f} (x)-\phi (x)|dx\\
&\le N\delta  (2+6(N+1)!)+N^2  (N^{-2\alpha}) N^{-2}\le 2N^{-2\alpha},
\end{split}
\end{equation*}
where the last inequality comes from the fact $\delta$ is small enough satisfying
\begin{equation}
\label{OneDDeltaCond1}
N\delta  (2+6(N+1)!)\le N^{-2\alpha}.
\end{equation}
Note that $f- (\phi-1)=f+1-\phi=\bar{f}-\phi$. Hence, $\phi-1\in \tn{NN} (\NNinput=1;\NNwidth=[2N,2N+1])$ and $\|f- (\phi-1)\|\le 2N^{-2\alpha}$. So, we finish the proof for the case $d=1$.

\beforestepskip
\noindent\textbf{ Step $2$}: The case $d>1$.    
\afterstepskip

The main idea is to project the $d$-dimensional problem into a one-dimensional one and use the results proved above. For any $N\in \N^+$, let $n=\lfloor N^{2/d}\rfloor$ and  $\delta$ be a sufficiently small positive number depending on $N$ and $d$, and satisfying \eqref{MultiDDeltaCond1}. We will divide the $d$-dimensional cube into $n^d$ small non-overlapping sub-cubes (see Figure \ref{fig:H1} for an illustration when $d=3$ and $n=3$), each of which is associated with a representative point, e.g., a vertex of the sub-cube. Due to the continuity, the target function $f$ can be represented by their values at the representative points. We project these representatives to one-dimensional samples via a ReLU FNN $\psi$ and construct a ReLU FNN $\bar{\phi}$ to fit them. Finally, the ReLU FNN $\phi$ on the $d$-dimensional space approximating $f$ can be constructed by $\phi=\bar{\phi}\circ\psi$. The precise construction can be found below.

By Lemma \ref{OneLayerProperty}, there exists $\psi_0\in \tn{NN} (\NNinput=1;\NNwidth=[2n])$ such that
\begin{itemize}
    \item $\psi_0 (1)=n-1$, and $\psi_0 (\tfrac{i}{n})=\psi_0 (\tfrac{i+1}{n}-\delta)=i$ for $i=0,1,\cdots,n-1$;
    \item $\psi_0$ is linear between any two adjacent points of $\{\tfrac{i}{n}: i=0,1,\cdots,n\}\cup\{\tfrac{i}{n}-\delta:i=1,2,\cdots,n\}$.
\end{itemize}
\vspace{5pt}
Define the projection map\footnote{The idea constructing such $\psi$ comes from the binary representation.} $\psi$ by
\begin{equation}
\label{DefinePsi}
\psi (\bm{x})=\sum_{i=1}^{d}\tfrac{1}{n^i}\psi_0 (x_i),\quad \tn{for}\ \bm{x}= [x_1,x_2,\cdots,x_d]^T\in [0,1]^d.
\end{equation}
Note that $\psi\in \tn{NN} (\NNinput=1;\NNwidth=[2dn])$. Given $f\in \tn{Lip} (\nu,\alpha,d)$, then $|f (\bm{x})|\le \sqrt{d}$ for ${\bm{x}}\in [0,1]^d$ since $f(\bm{0})=0$, $\nu=1$, and $\alpha\in(0, 1]$. Define $\bar{f}=f+\sqrt{d}$, then $0\le \bar{f} (\bm{x}) \le 2\sqrt{d}$ for ${\bm{x}}\in [0,1]^d$. Hence, we have
\begin{equation*}
\bigg\{\Big(\sum_{i=1}^d \tfrac{\theta_i}{n^i},\,\bar{f} (\tfrac{\bm{\theta}}{n})\Big):{\bm{\theta}}= [\theta_1,\theta_2,\cdots,\theta_d]^T\in\{0,1,\cdots,n-1\}^d \bigg\}\cup \{ (1,0)\}
\end{equation*}
as a set of $n^d+1$ samples of a one-dimensional function.  By Lemma \ref{OneLayerProperty}, $\exists\ \bar{\phi}\in \tn{NN} (\NNinput=1;\NNwidth=[2\lceil n^{d/2}\rceil, 2\lceil n^{d/2} \rceil+1])$ such that 
\begin{equation}
\label{DefinePhiBar}
\bar{\phi}\Big(\sum_{i=1}^d \tfrac{\theta_i}{n^i}\Big)=\bar{f}\left (\tfrac{\bm{\theta}}{n}\right),\quad\tn{for}\ {\bm{\theta}}= [\theta_1,\theta_2,\cdots,\theta_d]^T\in \{0,1,\cdots,n-1\}^d,
\end{equation}
and 
\begin{equation}
\sup_{t\in [0,1]}  |\bar{\phi} (t)|\le 6\sqrt{d}\left(\lceil n^{d/2} \rceil+1\right)!.
\end{equation}
Since the range of $\psi$ on $[0,1]^d$ is a subset of $[0,1]$, $\exists\ \phi\in \tn{NN}\left (\NNinput=d;\NNwidth=[2nd,2\lceil n^{d/2}\rceil,2\lceil n^{d/2}\rceil+1]\right)$
defined via $\phi(\bm{x})=\bar{\phi}\mycirc \psi(\bm{x})$  for ${\bm{x}}\in [0,1]^d$
such that
\begin{equation}
\label{UpperBoundPhiMultiD}
\sup_{{\bm{x}}\in [0,1]^d}  |\phi (\bm{x})|\le 6\sqrt{d}\left(\lceil n^{d/2} \rceil+1\right)!.
\end{equation}
Define $H_1=\cup_{j=1}^d\left\{{\bm{x}}= [x_1,x_2,\cdots,x_d]^T\in [0,1]^d:x_j\in \cup_{i=1}^{n}[\tfrac{i}{n}-\delta,\tfrac{i}{n}]\right\}$, which separates the $d$-dimensional cube into $n^d$ important sub-cubes as illustrated in Figure \ref{fig:H1}. To index these $d$-dimensional smaller sub-cubes, define $Q_{\bm{\theta}}=\big\{{\bm{x}}= [x_1,x_2,\cdots,x_d]^T\in [0,1]^d:x_i\in[\tfrac{\theta_i}{n},\tfrac{\theta_i+1}{n}-\delta], \ i=1,2,\cdots,d\big\}$
for each $d$-dimensional index ${\bm{\theta}}= [\theta_1,\theta_2,\cdots,\theta_d]^T\in \{0,1,\cdots,n-1\}^d$.  
 By \eqref{DefinePsi}, \eqref{DefinePhiBar}, and the definition of $\psi_0$, for any $\bm{x}= [x_1,x_2,\cdots,x_d]^T\in Q_{\bm{\theta}}$, we have $\phi (\bm{x})=\bar{\phi} (\psi (\bm{x}))
=\bar{\phi}\Big (\sum_{i=1}^d \tfrac{1}{n^i}\psi_0 (x_i)\Big)
=\bar{\phi}\Big (\sum_{i=1}^d \tfrac{1}{n^i}\theta_i\Big)
=\bar{f}\left (\tfrac{\bm{\theta}}{n}\right)$. Then 
\begin{equation}
\label{UpperBoundExceptBadRegionMultiD}
|\bar{f} (\bm{x})-\phi (\bm{x})|=|\bar{f} (\bm{x})-\bar{f}\left (\tfrac{\bm{\theta}}{n}\right)|\le  (\sqrt{d}/n)^{\alpha}, \quad \tn{for any}\ {\bm{x}}\in Q_{\bm{\theta}}.
\end{equation}

\begin{figure}[ht!]  \twofig[0.5]{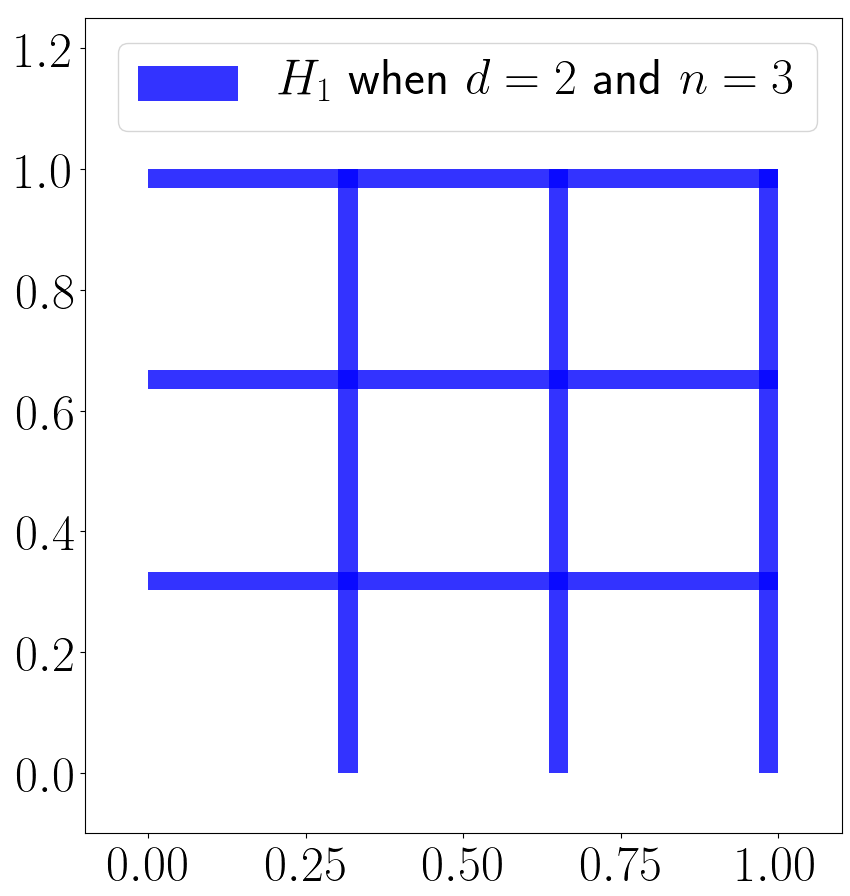}{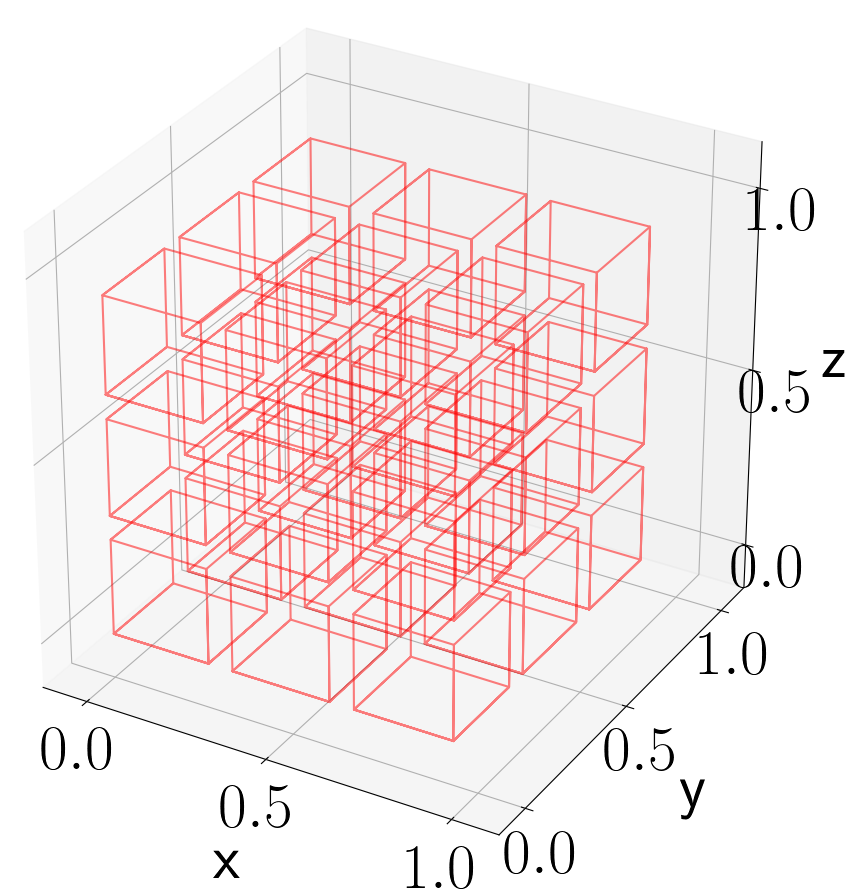}        
    \caption{ An illustration of $H_1$ and $n^d$ small non-overlapping sub-cubes that $H_1$ separates when $n=3$. (a) When $d=2$, $H_1$ in blue separates $[0,1]^2$ into $n^d=9$ sub-cubes. (b) When $d=3$, $H_1$ (no color) separates $[0,1]^3$ into $n^d=27$ sub-cubes in red.}
    \label{fig:H1}
\end{figure}

Because $\mu (H_1)\le dn\delta$, $[0,1]^d=\cup_{\bm{\theta}\in \{0,1,\cdots,n-1\}^d}Q_{\bm{\theta}} \cup H_1$, \eqref{UpperBoundPhiMultiD}, and \eqref{UpperBoundExceptBadRegionMultiD}, we have
\begin{equation*}
\begin{split}
\|\bar{f}-\phi\|_{L^1([0,1]^d)}
&=\int _{H_1} |\bar{f}-\phi| d{\bm{x}}+\int _{[0,1]^d\backslash H_1} |\bar{f}-\phi| d{\bm{x}}\\
&\le \mu (H_1) \Big(2\sqrt{d}+6\sqrt{d}\big(\lceil n^{d/2} \rceil+1\big)!\Big)
+\sum_{\bm{\theta}\in \{0,1,\cdots,n-1\}^d} \int_{Q_{\bm{\theta}}}|\bar{f}-\phi| d{\bm{x}}\\
&\le 2n\delta d\sqrt{d}\Big(1+3\big(\lceil n^{d/2} \rceil+1\big)!\Big)+\sum_{\bm{\theta}\in \{0,1,\cdots,n-1\}^d} (\sqrt{d}/n)^{\alpha} \mu (Q_{\bm{\theta}})\\
&\le  2d^{\alpha/2}n^{-\alpha},
\end{split}
\end{equation*}
where the last inequality comes from the fact that $\delta$ is small enough such that

\begin{equation}
\label{MultiDDeltaCond1}
2n\delta d\sqrt{d}\Big(1+3\big(\lceil n^{d/2} \rceil+1\big)!\Big)\le d^{\alpha/2}n^{-\alpha}.
\end{equation}

Note that $f- (\phi-\sqrt{d})=\bar{f}-\phi$. Hence, $\phi-\sqrt{d}\in \tn{NN} (\NNinput=d;\NNwidth=[2nd,2\lceil n^{d/2}\rceil,2\lceil n^{d/2}\rceil+1])$ and $\|f- (\phi-\sqrt{d})\|_{L^1([0,1]^d)}\le 2d^{\alpha/2}n^{-\alpha}$. Since $n=\lfloor N^{2/d}\rfloor$, we have $\lceil n^{d/2}\rceil\le N+1$. Therefore, 
\begin{equation*}
\begin{split}
\phi-\sqrt{d}&\in \tn{NN} (\NNinput=d;\NNwidth=[2d\lfloor N^{2/d}\rfloor,2N+2,2N+3])\\
\end{split}
\end{equation*}    
and 
\[\|f- (\phi-\sqrt{d})\|_{L^1([0,1]^d)}\le 2d^{\alpha/2}n^{-\alpha}= 2d^{\alpha/2}\lfloor N^{2/d}\rfloor^{-\alpha}\le 2d^{\alpha/2} (N^{2/d}/2)^{-\alpha}=2 (2\sqrt{d})^\alpha N^{-2\alpha/d},\] where the second inequality comes from the fact $\lfloor x\rfloor\ge \tfrac{x}{2}$ for $x\in [1,\infty)$. So, we finish the proof when $d>1$.
\end{proof}

Theorem \ref{thm:up} shows that for $f\in \tn{Lip}(\nu,\alpha,d)$ ReLU FNNs with two or three function compositions can achieve the approximation rate $\OO(\nu N^{-2\alpha/d})$. 
Following the same proof as in Theorem \ref{thm:up}, we can show that:

\begin{corollary}
\label{coll:upp}
$\forall\ m,n\in \N^+$, the closure of $\tn{NN}(\NNinput=1;\NNwidth=[2m,2n+1])$ contains $\tn{CPL}(mn+1)$ in the sense of $L^{1}$-norm.
\end{corollary}

An immediate implication of Corollary \ref{coll:upp} is that, for any function $f$ on $[0,1]$, if $f$ can be approximated via one-hidden-layer ReLU FNNs with an approximation rate $\OO(N^{-\eta})$ for any $\eta>0$, the rate can be improved to $\OO(N^{-2\eta})$ via one more  composition. 

\subsection{Asymptotic Unachievable  Approximation Rate}
\label{sub:low}
In Section \ref{sub:upp}, we have analyzed the approximation capacity of ReLU FNNs in the nonlinear approximation for general continuous functions by construction. In this section, we will show that the construction in Section \ref{sub:upp} is essentially and asymptotically tight via showing the approximation lower bound in Theorem \ref{thm:low} below.


\begin{theorem}
    \label{thm:low}
    For any $\ L\in \N^+$, $\rho>0$, and $C>0$, $\exists\ f\in \tn{Lip}(\nu,\alpha,d)$  with $\alpha\in(0,1]$,  for all $N_0>0$, there exists $N\ge N_0$ such that 
    \begin{equation*}
    \inf_{\phi\in \textnormal{NN}(\NNinput=d;\NNmaxwidth\le N;\NNlayer\le L)} \|\phi-f\|_{L^\infty([0,1]^d)}\ge C\nu N^{-(2\alpha/d+\rho)}.
    \end{equation*}
\end{theorem}

The proof of Theorem \ref{thm:low} relies on the  nearly-tight VC-dimension bounds of ReLU FNNs given in \cite{pmlr-v65-harvey17a} and is similar to Theorem $4$ of \cite{yarotsky2017}. Hence, we only sketch out its proof and a complete proof can be found in \cite{Shijun}. 

\begin{proof}
We will prove this theorem by contradiction.
Assuming that Theorem \ref{thm:low} is not true, we can show the following claim, which will lead to a contradiction in the end.
\begin{claim}
    \label{thm:lowClaim}
    There exist
     $ L\in \N^+$,  $\rho>0$, and $C>0$, $\forall\ f\in \tn{Lip} (\nu,\alpha,d)$ with $\alpha\in (0,1]$,  then $\exists\ N_0>0$, for all $N\ge N_0$,  there exists
    $\phi\in \tn{NN}(\NNinput=d;\NNmaxwidth\le N;\NNlayer\le L)$ such that
    \[
    \|f-\phi\|_{L^\infty([0,1]^d)} \le C\nu N^{-(2\alpha/d+\rho)}.
    \]
\end{claim}
If this claim is true, then we have a better approximation rate. So we need to disproof this claim in order to prove Theorem \ref{thm:low}.

Without loss of generality, we assume $\nu=1$; in the case of $\nu\neq 1$, the proof is similar by rescaling $f\in \tn{Lip}(\nu,\alpha,d)$ and FNNs with $\nu$. Let us denote the VC dimension of a function set $\mathcal{F}$ by $\tn{VCDim} (\mathcal{F})$. By \cite{pmlr-v65-harvey17a}, there exists $C_1>0$ such that
\begin{equation*}
\begin{split}
&\quad \tn{VCDim} \big(\tn{NN} (\NNinput=d;\NNmaxwidth\le N; \NNlayer\le L)\big)\\
& \le C_1\big ( (LN+d+2)(N+1)\big)L\ln \big ( (LN+d+2)(N+1)\big)\coloneqq b_u,
\end{split}
\end{equation*}
which comes from the fact the number of parameter of a ReLU FNN in $\NN(\NNinput=d;\NNmaxwidth\le N;\NNlayer\le L)$ is less than $(LN+d+2)(N+1)$.

One can estimate a lower bound of 
\begin{equation*}
\label{eqn:vcd}
\tn{VCDim} \big(\tn{NN} (\NNinput=d;\NNmaxwidth\le N;\NNlayer\le L)\big)
\end{equation*}
using Claim \ref{thm:lowClaim},
and this lower bound can be $b_\ell \coloneqq \lfloor N^{2/d+\rho/(2\alpha)}\rfloor ^d$, which is asymptotically larger than
\begin{equation*}
    \label{vcdUB}
    b_u:=C_1\big ( (LN+d+2)(N+1)\big)L\ln \big( (LN+d+2)(N+1)\big)=\OO(N^2\ln N),
\end{equation*}
 leading to a contradiction that disproves the assumption that ``Theorem \ref{thm:low} is not true". 
\end{proof}

Theorem \ref{thm:up} shows that the $N$-term approximation rate via two or three-hidden-layer ReLU FNNs can achieve $\OO(N^{-2\alpha/d})$, while Theorem \ref{thm:low} shows that the rate cannot be improved to $\OO(N^{-(2\alpha/d+\rho)})$ for any $\rho>0$. 
It was conjectured in the literature that function compositions can improve the approximation capacity exponentially. For general continuous functions, Theorem \ref{thm:low} shows that this conjecture is not true, i.e., if the depth of the composition is $L=\OO(1)$, the approximation rate cannot be better than $\OO(N^{-2\alpha/d})$, not to mention $\OO(N^{-L\alpha /d})$, which implies that adding one more layer cannot improve the approximation rate when $N$ is large and $L>2$.

Following the same proof as in Theorem \ref{thm:low}, we have the following corollary, which shows that the result in Corollary \ref{coll:upp} cannot be improved. 

\begin{corollary}
\label{coll:lowerBound}
$\forall\ \rho>0$, $C>0$ and $L\in \N^+$, $\exists\  N_0(\rho,C,L)$ such that for any integer $N\ge N_0$, $\tn{CPL}(CN^{2+\rho})$ is not contained in the closure of $\tn{NN}(\NNinput=1;\NNmaxwidth\le N;\NNlayer\le L)$ in the sense of $L^{\infty}$-norm.
\end{corollary}

\subsection{Approximation and Computation Efficiency in Parallel Computing}
\label{sub:oar}

In this section, we will discuss the efficiency of the $N$-term approximation via ReLU FNNs in parallel computing. This is of more practical interest than the optimal approximation rate purely based on the number of parameters of the nonlinear approximation since it is impractical to use FNNs without parallel computing in real applications. Without loss of generality, we assume $\nu=1$, $N\gg 1$, and $d\gg 1$. 

Let us summarize standard statistics of the time and memory complexity in parallel computing \cite{Kumar:2002:IPC:600009} in one training iteration of ReLU FNNs with $\OO(N)$ width and $\OO(L)$ depth using $m$ computing cores and $\OO(1)$ training data samples per iteration. Let $T_s(N,L,m)$ and $T_d(N,L,m)$ denote the time complexity in shared and distributed memory parallel computing, respectively. Similarly, $M_s(N,L,m)$ and $M_d(N,L,m)$ are the memory complexity for shared and distributed memory, respectively. $M_s(N,L,m)$ is the total memory requirement; while $M_d(N,L,m)$ is the memory requirement per computing core. Then

\begin{equation}
\label{eqn:Ts}
T_s(N,L,m)=\left\{
\begin{array}{ll}
\mathcal{O}\big(L(N^{2}/m+\ln\tfrac{m}{N})\big),   &   m\in [1,N^2], \\
\mathcal{O}(L\ln N) ,  &   m\in (N^2,\infty);   \\
\end{array}
\right.
\end{equation}
\begin{equation}
\label{eqn:Td}
T_d(N,L,m)=\left\{
\begin{array}{ll}
\mathcal{O}\big(L(N^2/m+t_s\ln m+\tfrac{t_w N}{\sqrt{m}}\ln m)\big),   &   m\in [1,N^2], \\
\mathcal{O}(L\ln N) ,  &   m\in (N^2,\infty);   \\
\end{array}
\right.
\end{equation}
\begin{equation}\label{eqn:Ms}
M_s(N,L,m)=\mathcal{O}(LN^{2}),\quad \tn{for all } m\in \N^+ ;
\end{equation}
and
\begin{equation}\label{eqn:Md}
M_d(N,L,m)=
\mathcal{O}(LN^{2}/m+1),\quad \tn{for all } m\in \N^+ ,
\end{equation}
where $t_s$ and $t_w$ are the ``start-up time'' and ``per-word transfer time'' in the data communication between different computing cores, respectively (see \cite{Kumar:2002:IPC:600009} for a detailed introduction). 
 
In real applications, a most frequently asked question would be: given a target function $f\in \tn{Lip}(\nu,\alpha,d)$, a target approximation accuracy $\epsilon$, and a certain amount of computational resources, e.g., $m$ computer processors, assuming the computer memory is enough, what is a good choice of FNN architecture we should use to reduce the running time of our computers? Certainly, the answer depends on the number of processors $m$ and ideally we hope to increase $m$ by a factor of $r$ to reduce the time (and memory in the distributed environment) complexity by the same factor $r$, which is the scalability of parallel computing. 
 
 We answer the question raised just above using FNN architectures that almost have a uniform width since the optimal approximation theory of very deep FNNs \cite{yarotsky18a} and this manuscript both utilize a nearly uniform width. Combining the theory in \cite{yarotsky18a} and ours, we summarize several statistics of ReLU FNNs in parallel computing in Table \ref{tab1} and \ref{tab2} when FNNs nearly have the same approximation accuracy. For shared memory parallel computing, from Table \ref{tab1} we see that: if computing resources are enough, shallower FNNs with $\OO(1)$ hidden layers require less and even exponentially less running time than very deep FNNs; if computing resources are limited, shallower FNNs might not be applicable or are slower, and hence very deep FNNs are a good choice. For distributed memory parallel computing, the conclusion is almost the same by Table \ref{tab2}, except that the memory limit is not an issue for shallower FNNs if the number of processors is large enough. In sum, if the approximation rate of very deep FNNs is not exponentially better than shallower FNNs, very deep FNNs are less efficient than shallower FNNs theoretically if computing resources are enough. 
 
 \begin{table}[ht]
    \label{tab1}
    \caption{The comparison of approximation and computation efficiency of different ReLU FNNs in shared memory parallel computing with $m$ processors when FNNs nearly have the same approximation accuracy. The analysis is asymptotic in  $N$ and is optimal up to a log factor ($N\gg d\gg 1$); ``running time'' in this table is the time spent on each training step with $\mathcal{O}(1)$ training samples. }
    \centering  
    \resizebox{\textwidth}{!}{ 
        \begin{tabular}{cccc} 
            \toprule
            &  $\tn{NN}(\NNwidth= [2d\lfloor N^{2/d}\rfloor,2N,2N])$ & $\tn{NN}(\NNwidth=[N]^L)$ & $\tn{NN}(\NNwidth=[2d+10]^N)$ \\
            
            \midrule
            accuracy $\epsilon$  &  $\mathcal{O}(\sqrt{d}N^{-2\alpha/d})$ &  $\mathcal{O}(C(d,L)N^{-2\alpha/d})$ & $\mathcal{O}(C(d)N^{-2\alpha/d})$\\
            
            number of weights  &   $\mathcal{O}(N^2)$ &  $\mathcal{O}(LN^2)$ & $\mathcal{O}(d^2N)$ \\    
            
            number of nodes &  $\mathcal{O}(N)$ & $\mathcal{O}(LN)$  & $\mathcal{O}(dN)$ \\
            
            running time for $m\in [1,(2d+10)^2]$ &  $\mathcal{O}(N^2/m)$ & $\mathcal{O}(LN^2/m)$
            & $\mathcal{O}(N(d^2/m+\ln \tfrac{m}{d}))$\\
            
            running time for $m\in ((2d+10)^2,N^2]$ &  $\mathcal{O}(N^2/m+\ln\tfrac{m}{N})$ & $\mathcal{O}(L(N^2/m+\ln \tfrac{m}{N}))$
            & $\mathcal{O}(N\ln d)$\\
            
            running time for $m\in (N^2,\infty)$&  $\mathcal{O}(\ln N)$ & $\mathcal{O}(L\ln N)$
            & $\mathcal{O}(N\ln d)$\\
            
            total memory &  $\mathcal{O}(N^2)$ &  $\mathcal{O}(LN^2)$ & $\mathcal{O}(d^2N)$ \\
            \bottomrule
        \end{tabular} 
    }
\end{table} 

\begin{table}[ht]
    \label{tab2}
    \caption{The comparison of approximation and computation efficiency of different ReLU FNNs in distributed memory parallel computing with $m$ processors when FNNs nearly have the same approximation accuracy. The analysis is asymptotic in $N$ and is optimal up to a log factor ($N\gg d\gg 1$); ``running time'' in this table is the time spent on each training step with $\mathcal{O}(1)$ training samples. } 
    \centering  
    \resizebox{\textwidth}{!}{ %
        \begin{tabular}{cccc} 
            \toprule
            &  $\tn{NN}(\NNwidth= [2d\lfloor N^{2/d}\rfloor,2N,2N])$ & $\tn{NN}(\NNwidth=[N]^L)$ & $\tn{NN}(\NNwidth=[2d+10]^N)$ \\
            
            \midrule
            accuracy $\epsilon$  &  $\mathcal{O}(\sqrt{d}N^{-2\alpha/d})$ &  $\mathcal{O}(C(d,L)N^{-2\alpha/d})$ & $\mathcal{O}(C(d)N^{-2\alpha/d})$\\
            
            number of weights  &   $\mathcal{O}(N^2)$ &  $\mathcal{O}(LN^2)$ & $\mathcal{O}(d^2N)$ \\    

            number of nodes &  $\mathcal{O}(N)$ & $\mathcal{O}(LN)$  & $\mathcal{O}(dN)$ \\

            running time for $m\in [1,(2d+10)^2]$ &  $\mathcal{O}(N^2/m+t_s\ln m+\tfrac{t_w N}{\sqrt{m}}\ln m)$ & $\mathcal{O}(L(N^2/m+t_s\ln m+\tfrac{t_w N}{\sqrt{m}}\ln m))$
            & $\mathcal{O}(N(d^2/m+t_s\ln m+\tfrac{t_w d}{\sqrt{m}}\ln m))$\\
            
            running time for $m\in ((2d+10)^2,N^2]$ &  $\mathcal{O}(N^2/m+t_s\ln m+\tfrac{t_w N}{\sqrt{m}}\ln m)$ & $\mathcal{O}(L(N^2/m+t_s\ln m+\tfrac{t_w N}{\sqrt{m}}\ln m))$
            & $\mathcal{O}(N\ln d)$\\
            
            running time for $m\in (N^2,\infty)$&  $\mathcal{O}(\ln N)$ & $\mathcal{O}(L\ln N)$
            & $\mathcal{O}(N\ln d)$\\
            
            memory per processor &  $\mathcal{O}(N^2/m+1)$ &  $\mathcal{O}(LN^2/m+1)$ & $\mathcal{O}(d^2N/m+1)$ \\
            \bottomrule
        \end{tabular} 
    }%
\end{table}

\section{Numerical Experiments}
\label{sec:numerical}

In this section, we provide two sets of numerical experiments to compare different ReLU FNNs using shared memory GPU parallel computing. The numerical results for distributed memory parallel computing would be similar. All numerical tests were conducted using Tensorflow and an NVIDIA P6000 GPU with 3840 CUDA parallel processing cores.

Since it is difficult to generate target functions $f\in \tn{Lip}(\nu,\alpha,d)$ with fixed $\nu$ and $\alpha$, we cannot directly verify the nonlinear approximation rate, but we are able to observe numerical evidence close to our theoretical conclusions. Furthermore, we are able to verify the running time estimates in Section \ref{sub:oar} and show that, to achieve the same theoretical approximation rate, shallow FNNs are more efficient than very deep FNNs.

In our numerical tests, we generate $50$ random smooth functions as our target functions using the algorithm in \cite{filip:hal-01944992} with a wavelength parameter $\lambda=0.1$ and an amplitude parameter $\sqrt{(2/\lambda)}$ therein. These target functions are uniformly sampled with $20000$ ordered points $\{x_i\}$ in $[0,1]$ to form a data set. 
The training data set consists of samples with odd indices $i$'s, while the test data set consists of samples with even indices $i$'s. The loss function is defined as the mean square error between the target function and the FNN approximant evaluated on training sample points. The ADAM algorithm \cite{ADAM} with a decreasing learning rate from $0.005$ to $0.0005$, a batch size $10000$, and a maximum number of epochs $20000$, is applied to minimize the mean square error. The minimization is randomly initialized by the ``normal initialization method"\footnote{See \url{https://medium.com/prateekvishnu/xavier-and-he-normal-he-et-al-initialization-8e3d7a087528}.}.
The test error is defined as the mean square error evaluated on test sample points. The training and test data sets are essentially the same in our numerical test since we aim at studying the approximation power of FNNs instead of the generalization capacity of FNNs. Note that due to the high non-convexity of the optimization problem, there might be chances such that the minimizers we found are bad local minimizers. Hence, we compute the average test error of the best $40$ tests among the total $50$ tests of each architecture. 

To observe numerical phenomena in terms of $N$-term nonlinear approximation,  in the first set of numerical experiments, we use two types of FNNs to obtain approximants: the first type has $L=\OO(1)$ layers with different sizes of width $N$; the second type has a fixed width $N=12$ with different numbers of layers $L$. Numerical results are summarized in Table \ref{tab:ShallowVSDeepFixWidth}. To observe numerical phenomena in terms of the number of parameters in FNNs, in the second set of numerical experiments, we use FNNs with the same number of parameters but different sizes of width $N$ and different numbers of layers $L$. Numerical results are summarized in Table \ref{tab:ShallowVSDeepFixParameter}. 

\begin{table}    
    \caption{ Comparison between $\tn{NN}(\NNinput=1;\NNwidth=[N]^L)$ and  $\tn{NN}(\NNinput=1;\NNwidth=[12]^N)$ for $N=32,64,128$ and $L=2,4,8$. ``Time'' in this table is the total running time spent on $20000$ training steps with training batch size $10000$, and the unit is second(s).} 
    \label{tab:ShallowVSDeepFixWidth}
    \centering 
    \resizebox{0.8\textwidth}{!}{ %
        \begin{tabular}{cccccccc} 
            \toprule
            $N$  & depth & width &  test error & improvement ratio & \#parameter & time & \\
            
            \midrule
            $32$  & $2$ & $32$ & 
            $8.06\times 10^{-2}$ &  -- &1153& $3.09\times 10^{1}$\\
            
            $32$  & $4$  & $32$ & 
            $3.98\times 10^{-4}$ & --&3265 &
            $3.82\times 10^{1}$\\
            
            $32$  & $8$  & $32$ & 
            $1.50\times 10^{-5}$ & --&7489 &
            $5.60\times 10^{1}$\\
            
            \midrule
            
            $32$  & $32$  & $12$ & 
            $1.29\times 10^{-3}$ & -- &4873 & $1.27\times 10^{2}$ \\
            
            \midrule
            $64$  & $2$ & $64$ & 
            $2.51\times 10^{-2}$ &  $3.21$ & 4353 &
            $3.45\times 10^{1}$\\
            
            $64$  & $4$  & $64$ & 
            $4.27\times 10^{-5}$ & $9.32$ & 12673 &
            $5.00\times 10^{1}$\\
            
            $64$  & $8$  & $64$ &  
            $2.01\times 10^{-6}$ &$7.46$&29313 &
            $7.91\times 10^{1}$\\
            
            \midrule
            
            $64$  & $64$  & $12$ & 
            $1.16\times 10^{-1}$ & $0.01$ &9865& $2.37\times 10^{2}$\\
            
            \midrule
            $128$  & $2$ & $128$ & 
            $2.04\times 10^{-3}$ &  $12.3$&16897 &
            $5.03\times 10^{1}$\\
            
            $128$  & $4$  & $128$ & 
            $1.05\times 10^{-5}$ & $4.07$&49921 &
            $8.21\times 10^{1}$\\
            
            $128$  & $8$  & $128$ &  
            $1.47\times 10^{-6}$ & $1.37$&115969 &
            $1.41\times 10^{2}$\\
            
            \midrule
            
            $128$  & $128$  & $12$ & 
            $3.17\times 10^{-1}$ & $0.37$ &19849 & $4.47\times 10^{2}$\\
            
            \bottomrule
        \end{tabular} 
    }%
\end{table}

By the last columns of Table \ref{tab:ShallowVSDeepFixWidth}, we verified that as long as the number of computing cores $m$ is larger than or equal to $N^2$, the running time per iteration of FNNs with $\OO(1)$ layers is $\OO(\ln N)$, while the running time per iteration of FNNs with $\OO(N)$ layers and $\OO(1)$ width is $\OO(N)$. By the last columns of Table \ref{tab:ShallowVSDeepFixParameter}, we see that when the number of parameters is the same, very deep FNNs requires much more running time per iteration than shallower FNNs and the difference becomes more significant when the number of parameters increases. Hence, very deep FNNs are much less efficient than shallower FNNs in parallel computing. 

Besides, by Table \ref{tab:ShallowVSDeepFixWidth} and \ref{tab:ShallowVSDeepFixParameter}, the test error of very deep FNNs cannot be improved if the depth is increased and the error even becomes larger when depth is larger. However, when the number of layers is fixed, increasing width can reduce the test error. More quantitatively, we define the \textit{improvement ratio} of an FNN with width $N$ and depth $L$ in Table \ref{tab:ShallowVSDeepFixWidth} as the ratio of the test error of an FNN in $\tn{NN}(\NNinput=1;\NNwidth=[N/2]^L)$ (or $\tn{NN}(\NNinput=1;\NNwidth=[N]^{L/2})$) over the test error of the current FNN in $\tn{NN}(\NNinput=1;\NNwidth=[N]^L)$. Similarly, the improvement ratio of an FNN with a number of parameters $W$ in Table \ref{tab:ShallowVSDeepFixParameter} is defined as the ratio of the test error of an FNN with the same type of architecture and a number of parameters $W/2$ over the test error of the current FNN. According to the improvement ratio in Table \ref{tab:ShallowVSDeepFixWidth} and \ref{tab:ShallowVSDeepFixParameter}, when $L=\OO(1)$, the numerical approximation rate in terms of $N$ is in a range between $2$ to $4$. Due to the high non-convexity of the deep learning optimization and the difficulty to generate target functions of the same class with a fixed order $\alpha$ and constant $\nu$, we cannot accurately verify the approximation rate. But the statistics of the improvement ratio can roughly reflect the approximation rate and the numerical results stand in line with our theoretical analysis. 

\begin{table}[ht]
    \caption{ Comparison between shallow FNNs and deep FNNs when the total number of parameters (\#parameter) is fixed. ``Time'' in this table is the total running time spent on $20000$ training steps with training batch size $10000$, and the unit is second(s).} 
    \label{tab:ShallowVSDeepFixParameter}
    \centering 
    \resizebox{0.74\textwidth}{!}{ 
        \begin{tabular}{cccccccc} %
            \toprule
            \#parameter & depth & width &  test error & improvement ratio & time \\
            
            \midrule
            $5038$  & $2$ & $69$ & 
            $1.13\times 10^{-2}$ &  -- & $3.84\times 10^{1}$\\
            
            $5041$  & $4$  & $40$ & 
            $1.65\times 10^{-4}$ & --&
            $3.80\times 10^{1}$\\
            
            $4993$  & $8$  & $26$ &  
            $1.69\times 10^{-5}$ & --&
            $5.07\times 10^{1}$\\
            
            \midrule
            
            $5029$  & $33$  & $12$ & 
            $4.77\times 10^{-3}$ & -- & $1.28\times 10^{2}$ \\
            
            \midrule
            $9997$  & $2$ & $98$ & 
            $4.69\times 10^{-3}$ &  $2.41$&
            $4.40\times 10^{1}$\\
            
            $10090$  & $4$  & $57$ & 
            $7.69\times 10^{-5}$ & $2.14$&
            $4.67\times 10^{1}$\\
            
            $9954$  & $8$  & $37$ &  
            $7.43\times 10^{-6}$ &$2.27$&
            $5.92\times 10^{1}$\\
            
            \midrule
            
            $10021$  & $65$  & $12$ & 
            $2.80\times 10^{-1}$ & $0.02$ & $2.31\times 10^{2}$\\
            
            \midrule
            $19878$  & $2$ & $139$ & 
            $1.43\times 10^{-3}$ &  $3.28$&
            $5.18\times 10^{1}$\\
            
            $20170$  & $4$  & $81$ & 
            $2.30\times 10^{-5}$ & $3.34$&
            $6.26\times 10^{1}$\\
            
            $20194$  & $8$  & $53$ &  
            $2.97\times 10^{-6}$ & $2.50$&
            $7.08\times 10^{1}$\\
            
            \midrule
            
            $20005$  & $129$  & $12$ &  
            $3.17\times 10^{-1}$ & $0.88$ & $4.30\times 10^{2}$\\
            
            \bottomrule
        \end{tabular} 
    }%
\end{table}

\section{Conclusions}
\label{sec:con}
 \textcolor{black}{We studied the approximation and computation efficiency of nonlinear approximation via compositions, especially when the dictionary $\mathcal{D}_L$ consists of deep ReLU FNNs with width $N$ and depth $L$. Our main goal is to quantify the best $N$-term approximation rate $\epsilon_{L,f}(N)$ for the dictionary $\mathcal{D}_L$ when $f$ is a H{\"o}lder continuous function. This topic is related to several existing approximation theories in the literature, but they cannot be applied to answer our problem. By introducing new analysis techniques that are merely based on the FNN structure instead of traditional approximation basis functions used in existing work, we have established a new theory to address our problem. }
 
 \textcolor{black}{In particular, for any function $f$ on $[0,1]$, regardless of its smoothness and even the continuity, if $f$ can be approximated using a dictionary when $L=1$ with the best $N$-term approximation rate $\epsilon_{L,f}=\OO(N^{-\eta})$, we showed that dictionaries with $L=2$ can improve the best $N$-term approximation rate to $\epsilon_{L,f}=\OO(N^{-2\eta})$. We also showed that for H{\"o}lder continuous functions of order $\alpha$ on $[0,1]^d$, the application of a dictionary with $L=3$ in nonlinear approximation can achieve an essentially tight best $N$-term approximation rate $\epsilon_{L,f}=\OO( N^{-2\alpha/d})$, and increasing $L$ further cannot improve the approximation rate in $N$. Finally, we showed that dictionaries consisting of wide FNNs with a few hidden layers are more attractive in terms of computational efficiency than dictionaries with narrow and very deep FNNs for approximating H{\"o}lder continuous functions if the number of computer cores is larger than $N$ in parallel computing. }

\textcolor{black}{Our results were based on the $L^1$-norm in the analysis of constructive approximations in Section \ref{sub:upp}; while we used $L^\infty$-norm in the unachievable approximation rate in Section \ref{sub:low}. In fact, we can define a new norm that is weaker than the $L^\infty$-norm and stronger than the $L^1$-norm such that all theorems in this paper hold in the same norm. The analysis based on the new norm can be found in \cite{Shijun} and our future work, and it shows that the approximation rate in this paper is tight in the same norm. Finally, although the current result is only valid for H{\"o}lder continuous functions, it is easy to generalize it to general continuous functions by introducing the moduli of continuity, which is also left as future work.}

\section*{Acknowledgments}
H. Yang thanks the Department of Mathematics at the National University of Singapore for the startup grant, the Ministry of Education in Singapore for the grant MOE2018-T2-2-147, NVIDIA Corporation for the donation of the Titan Xp GPU, and NSCC \cite{nscc} for extra GPU resource. We would like to acknowledge the editor and reviewers for their help in improving this manuscript.

\bibliographystyle{apacite}
\bibliography{references}

\end{document}

%% file: ex_shared.tex
\usepackage{breakcites}
\usepackage{apacite}

\usepackage{lipsum}
\usepackage{amsfonts}
\usepackage{graphicx}
\usepackage{epstopdf}
\usepackage{algorithmic}
\ifpdf
  \DeclareGraphicsExtensions{.eps,.pdf,.png,.jpg}
\else
  \DeclareGraphicsExtensions{.eps}
\fi

\newsiamremark{remark}{Remark}
\newsiamremark{hypothesis}{Hypothesis}
\crefname{hypothesis}{Hypothesis}{Hypotheses}
\newsiamthm{claim}{Claim}

\usepackage[utf8]{inputenc}

\usepackage{mathrsfs,amsmath,amsfonts,amssymb,mathtools,MnSymbol}
\usepackage{extarrows,xfrac,bm,nameref,dsfont,color,cancel}
\usepackage{srcltx}
\usepackage{fancybox}
\usepackage{tikz}
\usepackage{comment}
\usepackage{indentfirst}
\usepackage{booktabs}
\usepackage{textcomp,tabularx,colortbl,multirow,array}
\usepackage{tcolorbox}
\usepackage{graphicx}
\usepackage{float}
\usepackage{subfigure}
\graphicspath{{figures/}{experiments/}}
\makeatletter\@addtoreset{equation}{section}\makeatother
\renewcommand{\theequation}{\arabic{section}.\arabic{equation}}
\usepackage{enumerate}

\newcommand{\mycirc}{\scalebox{0.85}[0.85]{$\circ$}}
\newcommand{\R}{\ifmmode\mathbb{R}\else$\mathbb{R}$\fi}
\newcommand{\N}{\ifmmode\mathbb{N}\else$\mathbb{N}$\fi}
\newcommand{\Q}{\ifmmode\mathbb{Q}\else$\mathbb{Q}$\fi}

\newcommand{\beforestepskip}{\vspace{5pt}}
\newcommand{\afterstepskip}{\vspace{3pt}}
\newcommand{\tn}[1]{\textnormal{#1}}
\newcommand{\myto}{\mathop{\raisebox{-0pt}{\scalebox{2.6}[1]{$\longrightarrow$}}}}
\newcommand{\NN}{\tn{NN}}
\newcommand{\NNinput}{\tn{\#input}}
\newcommand{\NNwidth}{\, \tn{widthvec}}
\newcommand{\NNlayer}{\, \tn{\#layer}}

\newcommand{\NNmaxwidth}{\, \tn{maxwidth}}

\newcommand{\OO}{\mathcal{O}}
\newcommand{\xx}{{\bm{x}}}

\renewcommand{\epsilon}{\varepsilon}
\def\minmax{\mathop{\operator@font min\hspace*{2pt}max}\nolimits}
\def\maxmin{\mathop{\operator@font max\hspace*{2pt}min}\nolimits}

\DeclareMathOperator*{\argmin}{arg\,min}

\headers{Nonlinear Approximation via Compositions}{Z. Shen, H. Yang, and S. Zhang}

\title{Nonlinear Approximation via Compositions\thanks{
\funding{H. Yang was supported by the startup of the Department of Mathematics at the National University of Singapore.}}}

\author{Zuowei Shen\thanks{Department of Mathematics,  National University of Singapore
  (\email{matzuows@nus.edu.sg}).}
  \and Haizhao Yang\thanks{Department of Mathematics,  National University of Singapore
  (\email{haizhao@nus.edu.sg}).}
\and Shijun Zhang\thanks{Department of Mathematics,  National University of Singapore
  (\email{zhangshijun@u.nus.edu}).}}

\usepackage{makecell}
\usepackage{amsopn}

\makeatother
\newif\ifsubfigurename
\def\subfigurenamefalse{\let\ifsubfigurename\iffalse}
\def\subfigurenametrue{\let\ifsubfigurename\iftrue}
\subfigurenametrue
\newcounter{zsjsubfigure}
\makeatletter
\@addtoreset{zsjsubfigure}{figure}
\makeatother
\newcommand{\fig}[3][1]{\begin{minipage}{#1\textwidth}\centering\includegraphics[width=#3\textwidth]{#2}
\ifsubfigurename\ \\ \stepcounter{zsjsubfigure} (\alph{zsjsubfigure})\fi%
\end{minipage}} 

\newcommand{\twofig}[3][1]{\par \vspace*{3.5pt} \fig[0.49]{#2}{#1}\hspace*{0.02\textwidth}\fig[0.49]{#3}{#1}\vspace*{3.5pt}}
